\documentclass[11pt,a4paper]{article}
\usepackage{fancyhdr}
\usepackage{amsmath}
\usepackage{bm}
\usepackage{subfig}
\usepackage{esint}
\usepackage{amsfonts}
\usepackage{amsthm}
\usepackage{amssymb}
\usepackage{amsbsy}
\usepackage{verbatim}
\usepackage{graphicx}
\usepackage{url}
\usepackage{mathrsfs}
\usepackage[hmargin = 1in,vmargin=.75in]{geometry}
\usepackage{color}
\usepackage{amstext}
\usepackage{stmaryrd}
\usepackage{enumerate}
\usepackage{algpseudocode}
\usepackage{algorithm}
\usepackage{pifont}
\usepackage{prettyref}
\usepackage{hyperref}
\hypersetup{
    colorlinks = true,
}

\font\msbm=msbm10

\numberwithin{equation}{section}

\theoremstyle{plain}
\newtheorem{theorem}{Theorem}[section]
\newtheorem{lemma}[theorem]{Lemma}

\newtheorem{proposition}[theorem]{Proposition}
\newtheorem{definition}{Definition}[section]

\def\mathbb#1{\hbox{\msbm{#1}}}

\newcommand{\tr}{\operatorname{Tr}}

\newcommand{\ba}{\boldsymbol{a}}
\newcommand{\bb}{\boldsymbol{b}}

\newcommand{\be}{\boldsymbol{e}}
\newcommand{\bg}{\boldsymbol{g}}
\newcommand{\bh}{\boldsymbol{h}}

\newcommand{\bp}{\boldsymbol{p}}

\newcommand{\bs}{\boldsymbol{s}}
\newcommand{\bt}{\boldsymbol{t}}
\newcommand{\bu}{\boldsymbol{u}}
\newcommand{\bv}{\boldsymbol{v}}
\newcommand{\bw}{\boldsymbol{w}}
\newcommand{\bx}{\boldsymbol{x}}
\newcommand{\by}{\boldsymbol{y}}
\newcommand{\bz}{\boldsymbol{z}}
\newcommand{\bone}{\boldsymbol{1}}

\newcommand{\balpha}{\boldsymbol{\alpha}}
\newcommand{\bbeta}{\boldsymbol{\beta}}
\newcommand{\bgamma}{\boldsymbol{\gamma}}
\newcommand{\bmu}{\boldsymbol{\mu}}

\newcommand{\BA}{\boldsymbol{A}}
\newcommand{\BB}{\boldsymbol{B}}
\newcommand{\BC}{\boldsymbol{C}}

\newcommand{\BH}{\boldsymbol{H}}
\newcommand{\BI}{\boldsymbol{I}}
\newcommand{\BJ}{\boldsymbol{J}}

\newcommand{\BP}{\boldsymbol{P}}
\newcommand{\BQ}{\boldsymbol{Q}}

\newcommand{\BS}{\boldsymbol{S}}

\newcommand{\BU}{\boldsymbol{U}}
\newcommand{\BV}{\boldsymbol{V}}
\newcommand{\BW}{\boldsymbol{W}}
\newcommand{\BX}{\boldsymbol{X}}
\newcommand{\BY}{\boldsymbol{Y}}
\newcommand{\BZ}{\boldsymbol{Z}}

\newcommand{\BPhi}{\boldsymbol{\Phi}}

\newcommand{\BPi}{\boldsymbol{\Pi}}

\newcommand{\bvarphi}{\boldsymbol{\varphi}}

\newcommand{\btheta}{\boldsymbol{\theta}}

\newcommand{\bzero}{\boldsymbol{0}}

\newcommand{\BSigma}{\boldsymbol{\Sigma}}

\newcommand{\pa}{\partial}

\newcommand{\I}{\boldsymbol{I}}
\newcommand{\RR}{\mathbb{R}}
\newcommand{\lag}{\langle}
\newcommand{\rag}{\rangle}

\newcommand{\eps}{\epsilon}

\newcommand*\diff{\mathop{}\!\mathrm{d}}
\DeclareMathOperator{\Tr}{Tr}
\DeclareMathOperator{\argmin}{argmin}

\DeclareMathOperator{\supp}{supp}

\DeclareMathOperator{\E}{\mathbb{E}}

\DeclareMathOperator{\diag}{diag}

\DeclareMathOperator{\ReLU}{ReLU}

\renewcommand{\Pr}{\mathbb{P}}

\long\def\\#1//{}

\definecolor{xl}{RGB}{200,50,120}

\begin{document}
\title{\bf Beyond Unconstrained Features: Neural Collapse \\ for Shallow Neural Networks with General Data}
\author{Wanli Hong\thanks{Shanghai Frontiers Science Center of Artificial Intelligence and Deep Learning, New York University Shanghai, China. S.L. and W.H. is (partially) financially supported by the National Key R\&D Program of China, Project Number 2021YFA1002800, National Natural Science Foundation of China (NSFC) No.12001372, Shanghai Municipal Education Commission (SMEC) via Grant 0920000112, and NYU Shanghai Boost Fund. W.H. is also supported by NYU Shanghai Ph.D. fellowship and acknowledges the NSF/NRT support.}~\thanks{Center for Data Science, New York University.},~ Shuyang Ling$^*$}

\maketitle

\begin{abstract}
    Neural collapse (${\cal NC}$) is a phenomenon that emerges at the terminal phase of the training (TPT) of deep neural networks (DNNs). The features of the data in the same class collapse to their respective sample means and the sample means exhibit a simplex equiangular tight frame (ETF). In the past few years, there has been a surge of works that focus on explaining why the ${\cal NC}$ occurs and how it affects generalization. Since the DNNs are notoriously difficult to analyze, most works mainly focus on the unconstrained feature model (UFM). While the UFM explains the ${\cal NC}$ to some extent, it fails to provide a complete picture of how the network architecture and the dataset affect ${\cal NC}$. In this work, we focus on shallow ReLU neural networks and try to understand how the width, depth, data dimension, and statistical property of the training dataset influence the neural collapse. We provide a complete characterization of when the ${\cal NC}$ occurs for two or three-layer neural networks. For two-layer ReLU neural networks, a sufficient condition on when the global minimizer of the regularized empirical risk function exhibits the ${\cal NC}$ configuration depends on the data dimension, sample size, and the signal-to-noise ratio in the data instead of the network width. For three-layer neural networks, we show that the ${\cal NC}$ occurs as long as the first layer is sufficiently wide. Regarding the connection between ${\cal NC}$ and generalization, we show the generalization heavily depends on the SNR (signal-to-noise ratio) in the data: even if the ${\cal NC}$ occurs, the generalization can still be bad provided that the SNR in the data is too low. Our results significantly extend the state-of-the-art theoretical analysis of the ${\cal NC}$ under the UFM by characterizing the emergence of the ${\cal NC}$ under shallow nonlinear networks and showing how it depends on data properties and network architecture.
\end{abstract}

\section{Introduction}

Deep neural networks have achieved tremendous success in the past few years in a variety of applications~\cite{HZRS16,LBH15}.
However, the mystery behind deep neural networks and deep learning remains far behind its significant applications in practice. In this work, we will focus on a phenomenon called neural collapse $({\cal NC})$ that was observed in~\cite{PHD20} that some particular structures emerge in the feature representation layer and the classification layer of DNNs in the TPT regime for classification tasks when the training dataset is balanced. 

To make our discussion more precise, we consider an $L$-layer feedforward neural network in the form of 
\begin{equation}\label{def:dnn}
\bh_{\ell+1}(\bx) = \sigma(\BW_{\ell}^{\top}\bh_{\ell}(\bx) + \bb_{\ell}),~~1\leq\ell\leq L-1,
\end{equation}
where $\BW_{\ell}$ and $\bb_{\ell}$ are the weight and bias on the $\ell$-th layer, and the output after $L-1$ layers goes through the last linear layer gives $\BW^{\top}_L \bh_L(\bx)$ which is used for classification. Here $\sigma(\cdot)$ is a nonlinear activation function such as ReLU activation and sigmoid function.

From now on, we always set $\BW : = \BW_L$, i.e., the linear classifier on the last layer, and let $\bh_{\btheta}(\bx) := \bh_L(\bx)$ be the feature map of the data point $\bx$ where $\btheta$ represents all the parameters $\{(\BW_{\ell},\bb_{\ell})\}_{\ell=1}^{L-1}$. Given the training data $\{(\bx_i,\by_i)\}_{i=1}^N$, the model training reduces to the empirical risk minimization:
\[
R_N(\BW,\btheta) := \frac{1}{N} \sum_{i=1}^N \ell( \BW^{\top}\bh_{\btheta}(\bx_i), \by_i) 
\]
where $\BW\in\RR^{D\times K}$, $\bh_{\btheta}(\bx): \RR^{d}\rightarrow \RR^{D}$ is a feature map and $\ell(\cdot,\cdot)$ is a loss function such $\ell_2$-loss and cross-entropy (CE) loss.

Suppose the training dataset consists of $K$ classes with equal size $\{\bx_{ki}\}_{1\leq i\leq n,1\leq k\leq K}$, i.e., each class contains exactly $n$ points, and the label $\by_k = \be_k$ is a one-hot vector. Then at the nearly final stage of training, the linear classifier and features exhibit the following structures~\cite{PHD20} ${\cal NC}_1$-${\cal NC}_3$:

\begin{itemize}
\item $\mathcal{NC}_1$ variability collapse: the features of the samples from the same class converge to their mean feature vector, i.e., 
\[
\bh_{ki}:=\bh_{\theta}(\bx_{ki}) \longrightarrow \bar{\bh}_k : = \frac{1}{n}\sum_{i=1}^n \bh_{\theta}(\bx_{ki}),~~~1\leq i\leq n,~1\leq k\leq K,
\]
as the training evolves where $\bh_{ki}$ is the feature of the $i$-th data point in the $k$-th class.

\item $\mathcal{NC}_2$ convergence to the simplex ETF: these feature vectors form an equiangular tight frame (ETF), i.e., they share the same pairwise angles and length; 
\[
[\lag \bar{\bh}_k, \bar{\bh}_{k'}\rag]_{k,k'} \propto \I_K - \BJ_K/K
\]
where $\I_K$ and $\BJ_K$ are the $K\times K$ identity and constant ``1" matrices respectively.

\item $\mathcal{NC}_3$ convergence to self-duality: the weight of the linear classifier converges to the corresponding feature mean (up to a scalar product):
\[
\BW^{\top}\BW \propto \bar{\BH}^{\top}\bar{\BH}
\]
where $\bar{\BH}\in\RR^{D\times K}$ consists of the mean feature vectors $\{\bar{\bh}_k\}_{k=1}^K.$

\end{itemize}

These empirical findings have sparked a series of theoretical works that try to explain the emergence of the ${\cal NC}$ in the training of DNNs.
The key question is to understand when neural collapse occurs and how the neural collapse is related to the generalization. 
Starting from~\cite{EW22,FHLS21,LS22,MPP22}, a thread of works considers the unconstrained feature model (UFM) or the layer peeled model to explain the emergence of neural collapse. The rationale behind UFM relies on the universal approximation theorem~\cite{C89,HSW89}: over-parameterize deep neural networks have an exceptional power to approximate many common function classes. Therefore, the expressiveness of the feature map $\bh_{\btheta}(\bx)$ is extremely powerful so that it can be replaced by arbitrary vectors. 
Various versions of UFMs with different loss functions and regularizations are proposed in these works~\cite{EW22,MPP22,ZDZ21,FHLS21,DNT23,LS22,TB22,MPP22,YWZB22,ZLD22} including the ${\cal NC}$ on imbalanced datasets~\cite{FHLS21,WL24,DNT23,DTNH24,TKV22}. All the aforementioned works manage to find that the global minimizers of the empirical risk function under the UFMs and the global minimizers indeed match the characterization of ${\cal NC}$ proposed in~\cite{PHD20} such as variability collapse ${\cal NC}_1$ and the convergence to simplex ETF ${\cal NC}_{2-3}$. 
There are also some recent works trying to extend the UFMs to deep linear neural networks~\cite{DTO+23,GK24} and also nonlinear ReLU network~\cite{DTNH24,NLL+23,SMC24,TB22}. In particular,~\cite{TB22} shows that for two-layer ReLU networks with nonlinear unconstrained features, the global minimizer to the empirical risk function with $\ell_2$-loss matches the ${\cal NC}$ configuration. In~\cite{SMC24}, the authors prove that for a binary classification task with deep unconstrained feature networks, the global minimum exhibits all the typical properties of the ${\cal NC}$. 
Another streamline of works focuses on analyzing the dynamical side of neural collapse: whether the optimization algorithm such as gradient descent converges to the ${\cal NC}$~\cite{PL21a,HPD21,JLZ+22,MPP22,THN23,RB22}. Also, the ${\cal NC}$ has been studied in the domain of transfer learning~\cite{LLZ22,GGH21} and the relation between ${\cal NC}$ and neural tangent kernel is discussed in~\cite{SWG+24}. For a recent review on the topic of neural collapse, the readers might refer to~\cite{KRA22}.

As the ${\cal NC}$ is rather challenging to understand for deep neural networks, the UFMs have been a great surrogate to the complicated DNNs and make the theoretical analysis much less difficult. However, the unconstrained feature model also has its obvious disadvantage: by replacing $\bh(\bx)$ with any vector, the resulting feature map actually has nothing to do with the alignment of training data and labels. Consequently, the study of the $\mathcal{NC}$ under UFMs is more or less equal to an analysis of optimization phenomenon. In addition, the relation between input data and neural collapse remains unexplored yet but is the key to understanding this intriguing phenomenon of ${\cal NC}$~\cite{PHD20,YSH23}. Therefore, the UFMs do not seem to be suitable models to study the connection between the generalization of DNNs and ${\cal NC}$~\cite{HBN22}.

To overcome the obvious limitations of UFMs, we take one step further to study the ${\cal NC}$ under a more realistic setting: we focus on the shallow ReLU neural networks (with the number of layers equal to $L=2$ or 3) with a given dataset $\{(\bx_i, \by_i)\}_{i=1}^N$. In other words, we will study the ${\cal NC}$ under a nonlinear feature model and more importantly, in the presence of a general dataset. 
More precisely, our work aims to address a few important questions regarding the emergence of {\cal NC} and its connection to generalization:
\begin{equation}
\text{\em Does the neural collapse occur for a sufficiently wide shallow network? }\tag{Question 1}
\end{equation}
It is an interesting question since we believe that the family of two-layer neural networks already has the property of universal approximation. Therefore, it is crucial to see if the universal approximation of a two-layer neural network implies neural collapse, i.e., whether the global minimizer to the regularized empirical risk minimization equals the ETF configuration.

\begin{equation}\tag{Question 2}
\text{\em Is the ${\cal NC}$ more likely to occur on a dataset with cluster structures?} 
\end{equation}
The motivation behind this question arises from a simple observation: if the datasets are highly separated, i.e., each data point is very close to its mean and far away from the data in other classes, then intuitively the ${\cal NC}$ is more likely to happen than on purely random noise. Therefore, it is natural to ask how the separation (which can be quantified by the signal-to-noise ratio) determines the emergence of ${\cal NC}$. 

\begin{equation}
\text{\em Does the emergence of the ${\cal NC}$ necessarily imply excellent generalization?}
\tag{Question3}
\end{equation}
 Intuitively, ${\cal NC}$ may help the generalization as the ${\cal NC}$ maps data in the same class to a single point, and the classifier is simplified to a nearest class-mean decision rule, which is related to the max-margin classifier~\cite{PHD20} and implicit bias~\cite{BKV23,SHNG18}. Thus, it is important to address the connection between ${\cal NC}$ and generalization.

We summarize our main contribution as follows: we provide a new proof for the ${\cal NC}$ under positive unconstrained ReLU feature model via convex optimization that is very versatile and can potentially apply to many other unconstrained feature settings. More importantly, we derive sufficient conditions for the ${\cal NC}$ under both general datasets and Gaussian mixture model (GMM) for cross-entropy loss and $\ell_2$-loss: (a) we provide a complete characterization of when the ${\cal NC}$ occurs for two-layer neural networks, i.e., when the global minimizer exhibits the ${\cal NC}$ configuration.
For a two-layer neural network, even if it is sufficiently wide, the ${\cal NC}$ may not occur if the dimension of input data is not sufficiently large. (b) If the data dimension is moderate, whether the ${\cal NC}$ occurs mainly depends on the SNR in the data instead of the network width, which is fully elucidated under the GMM as a benchmark example. For three-layer neural networks, we show that the ${\cal NC}$ occurs as long as the first layer is sufficiently wide. (c) In addition, we show that under the Gaussian mixture model, the generalization mainly depends on the SNR in the dataset even if the ${\cal NC}$ occurs. Our results significantly extend the state-of-the-art works that mainly focus on the unconstrained feature models.

\subsection{Notation}
We let boldface letter $\BX$ and $\bx$ be a matrix and a vector respectively; $\BX^{\top}$ and $\bx^{\top}$ are the transpose of $\BX$ and $\bx$ respectively. The matrices $\I_n$, $\BJ_n$, and $\be_k$ are the $n\times n$ identity matrix, a constant matrix with all entries equal to 1, and the one-hot vector with $k$-th entry equal to 1. Let
\begin{equation}\label{def:ck}
\BC_K : = \I_K - \BJ_K/K
\end{equation}
be the $K\times K$ centering matrix. 
For any vector $\bx,$ $\diag(\bx)$ denotes the diagonal matrix whose diagonal entries equal $\bx.$ 
For any matrix $\BX$, we let $\|\BX\|$, $\|\BX\|_F$, and $\|\BX\|_*$ be the operator norm, Frobenius form, and nuclear norm. 

\subsection{Organization}

The following sections are organized in the following way. In Section~\ref{s:main}, we start with the model setting and then present our main theoretical results. In Section~\ref{s:numerics}, we provide numerical experiments to support our theoretical findings, and Section~\ref{s:proof} justifies all our theorems.

\section{Preliminaries and main theorems}\label{s:main}

In this work, we focus on the empirical risk minimization with activation regularization:
\begin{equation}\label{eq:model}
R_N(\BW,\btheta) = \frac{1}{N} \sum_{i=1}^N \ell( \BW^{\top}\bh_{\btheta}(\bx_i), \by_i) + \frac{\lambda_W}{2}\|\BW\|_F^2 + \frac{\lambda_H}{2} \|\BH_{\btheta}(\BX)\|_F^2
\end{equation}
where $\BH_{\btheta}(\BX) = [\bh(\bx_1),\cdots,\bh(\bx_N)]\in\RR^{D\times N}$ is the aggregation of all features. 

Throughout our discussion, we assume that the dataset is balanced, i.e., 
\begin{equation}\label{def:Y}
\BY := \I_K\otimes \bone_n^{\top} = [\underbrace{\be_1,\cdots,\be_1}_{n\text{ times}},\cdots,\underbrace{\be_K,\cdots,\be_K}_{n\text{ times}}]\in\RR^{K\times Kn}
\end{equation}
is the label matrix where each column represents the label of the corresponding data point, and each class contains $n$ samples. In particular, the mean feature matrix $\bar{\BH}$ is given by
\begin{equation}\label{def:barH}
\bar{\BH} := \frac{1}{n} \BH\BY^{\top}\in\RR^{D\times K},~~\bar{\bh}_k : = \frac{1}{n}\sum_{i=1}^n \bh_{ki}
\end{equation}
where $\bar{\bh}_k$ is the $k$-th column of $\bar{\BH}.$ The regularization is crucial as it has been shown empirically that without the regularization, the ${\cal NC}$ will not happen. In particular,~\cite{RB22} has shown that without regularization, the DNN will interpolate the data but the features will not exhibit ${\cal NC}$ phenomenon. The reason for choosing activation regularization instead of weight decay is due to the simplicities of analyzing the activation regularization. We will leave the study of the role of weight decay to our future research agenda. 
Before proceeding, we will need to formalize a concept about the occurrence of neural collapse.

\begin{definition}[\bf Neural collapse occurs] 
We say the neural collapse occurs for the neural network $f_{\BW,\btheta}(\bx) = \BW^{\top}\bh_{\btheta}(\bx)$ with data $\{(\bx_{ki},\by_k)\}_{1\leq i\leq n,1\leq k\leq K}$ if there exists $(\BW,\btheta)$ such that it is equal to the global minimizer of~\eqref{eq:model} and also it holds
\[
\BH_{\btheta}(\BX) = \bar{\BH}_{\theta}(\BX)\otimes \bone_n^\top,~~~\bar{\BH}_{\theta}(\BX)^\top\bar{\BH}_{\theta}(\BX)\propto \BI_K, ~~~\BW^{\top}\BW \propto \I_K - \BJ_K/K
\]
i.e., the feature matrix $\BH_{\btheta}(\BX)$ satisfies the within-class variability collapse and the mean feature $\bar{\BH}_{\theta}(\BX)$ in~\eqref{def:barH} converges to an orthogonal frame.
\end{definition}
Here are two important remarks about the definition. (a) The reason why we use the convergence of the mean features to an orthogonal frame instead of a simplex ETF is because of the nonnegativity of the ReLU feature discussed in this work, also see~\cite{DTNH24,NLL+23,TB22}; (b) This definition only concerns the global minimizer to~\eqref{def:dnn} and its ${\cal NC}$ properties, but we do not discuss whether an algorithm (gradient descent or SGD) converges to a global minimizer, which is beyond the scope of this work and will be investigated in the future.
As~\eqref{eq:model} is quite challenging to analyze in general, we will focus on the following special yet non-trivial models. 

\paragraph{Two-layer and three-layer ReLU network}
We consider the two-layer bias-free neural network with ReLU activation function $\sigma_{\ReLU}(x) = \max\{x,0\}$, i.e., $L=2$ or $3$ in~\eqref{def:dnn}. For two-layer neural networks, we have the empirical risk function as follows:
\begin{equation}\label{def:dnn2}
R_N(\BW,\BW_1) = \frac{1}{N} \sum_{i=1}^N \ell( \BW^{\top}\sigma_{\ReLU}(\BW_1^{\top}\bx_i), \by_i) + \frac{\lambda_W}{2}\|\BW\|_F^2 + \frac{\lambda_H}{2} \|\sigma_{\ReLU}(\BW_1^{\top}\BX)\|_F^2
\end{equation}
where $\BW_1\in\RR^{d\times D}$, $\BW\in\RR^{D\times K}$, and $\BX\in\RR^{d\times N}.$ 
In addition, we also consider the three-layer counterpart whose empirical risk function is given by
\begin{equation}\label{def:dnn3}
  \begin{aligned}
R_N(\BW,\BW_1,\BW_2) = &\frac{1}{N} \sum_{i=1}^N \ell( \BW^{\top}\sigma_{\ReLU}(\BW_2^{\top}\sigma_{\ReLU}(\BW_1^{\top}\bx_i)), \by_i) \\
&~~~+\frac{\lambda_W}{2}\|\BW\|_F^2 + \frac{\lambda_H}{2} \|\sigma_{\ReLU}(\BW_2^{\top}\sigma_{\ReLU}(\BW_1^{\top}\BX))\|_F^2
  \end{aligned}
\end{equation}
where $\BW_1\in\RR^{d\times d_1}$, $\BW_2\in\RR^{d_1\times D}$, $\BW\in\RR^{D\times K}$, and $\BX\in\RR^{d\times N}.$
The two-layer neural network (two-layer NN) has become one important model to study from a theoretical perspective as it is possibly the simplest nonlinear neural network. The convergence of gradient descent and SGD in training two-layer neural networks has been discussed in~\cite{LY17,DZPS18} as well as in~\cite{MMN18,RV22} by using the mean-field analysis of shallow networks. Despite the shallowness, two-layer neural networks have powerful approximation properties~\cite{C89,HSW89}: the universal approximation theorem, which motivates the study of the unconstrained feature model or the layer-peeled model~\cite{FHLS21,ZLD22}. This inspires us to understand the ${\cal NC}$ phenomenon in shallow networks, especially when the global minimizers to~\eqref{def:dnn3} include the ${\cal NC}$ configuration.

\paragraph{Unconstrained positive feature model }
One obvious difficulty to determine when the ${\cal NC}$ occurs comes from the nonlinearity of~\eqref{def:dnn} even if $L=2$ or $3$. Therefore, in most state-of-the-art literature, a significant simplification of the nonlinear model~\eqref{def:dnn} is to assume the feature matrix $\BH = \sigma_{\ReLU}(\BW_1^{\top}\BX)$ is unconstrained, i.e., $\BH$ is any $D\times N$ nonnegative matrix.
We will also start with analyzing whether the ${\cal NC}$ occurs under the unconstrained positive feature model (UPFM).
Under $\ell_{CE}$-loss function and unconstrained positive feature model, we have
\begin{equation}\label{def:upfm_ce}
R_N(\BW,\BH) = \frac{1}{N} \ell_{CE}(\BW^{\top}\BH,\BY) + \frac{\lambda_W}{2}\|\BW\|_F^2 + \frac{\lambda_H}{2}\|\BH\|_F^2
\end{equation}
subject to $\BH\geq 0$ where $\BY$ is defined in~\eqref{def:Y},
\[
\ell_{CE}(\bz, \be_k) = \log \sum_{j=1}^n e^{z_j} - z_k
\]
and $\be_k$ is a one-hot vector.
The counterpart under $\ell_2$-loss is
\begin{equation}\label{def:upfm_l2}
R_N(\BW,\BH) = \frac{1}{2N}\|\BW^{\top}\BH -\BY\|_F^2 + \frac{\lambda_W}{2}\|\BW\|_F^2 + \frac{\lambda_H}{2} \|\BH\|_F^2
\end{equation}
subject to $\BH\geq 0$. 
The study of the ${\cal NC}$ on~\eqref{def:upfm_ce} and~\eqref{def:upfm_l2} is the first step towards to understanding the ${\cal NC}$ for two/three-layer neural networks~\eqref{def:dnn2} and~\eqref{def:dnn3}. In fact, analyzing~\eqref{def:upfm_l2} is quite straightforward as it is directly related to the singular value thresholding. For~\eqref{def:dnn2}, it is slightly more complicated but we will provide proof of the ${\cal NC}$ under UPFM via convex relaxation.

\subsection{${\cal NC}$ under unconstrained feature model with general data}
Now we present our first theorem that characterizes the ${\cal NC}$ under the unconstrained positive feature model for both $\ell_2$- and cross-entropy loss.

\begin{theorem}[\bf Neural collapse under unconstrained positive feature models]\label{thm:upfm}
Under the unconstrained positive feature model with a balanced dataset, the following holds.
\begin{enumerate}[(a)]
\item The global minimizer to~\eqref{def:upfm_ce} is given by
\begin{equation}\label{eq:barH_ce_NC1}
\begin{aligned}
& \BH^{\top}\BH = b\BY^{\top}\BY=b \I_K\otimes\BJ_n,~~~~~\BW^{\top}\BW = \frac{a^2}{b}(\I_K - \BJ_K/K), \\
& \BW^{\top}\BH = a (\I_K - \BJ_K/K)\BY,
\end{aligned}
\end{equation}
i.e., the mean feature $\bar{\BH}$ is nonnegative and satisfies 
\begin{equation}\label{eq:barH_ce}
\bar{\BH}^{\top}\bar{\BH} = b\I_K,~~~\BW^{\top}\bar{\BH} = a(\I_K - \BJ_K/K),~~~ \BW = \frac{a}{b}\bar{\BH}(\I_K - \BJ_K/K)
\end{equation}
where $\BY$ is defined in~\eqref{def:Y}, 
\[
b = \sqrt{\frac{K-1}{nK}\cdot\frac{\lambda_W}{\lambda_H}} a,~~a = \max\left\{\log \left( (K-1) \left[ \sqrt{ \frac{1}{nK(K-1)}\cdot\frac{1}{\lambda_H\lambda_W} } - 1\right]\right), 0\right\}
\]
In particular, $a > 0$ if 
\[
\sqrt{ \frac{n(K-1)}{K} } > Kn \lambda
\]
where $\lambda = \sqrt{\lambda_W\lambda_H}.$ Here~\eqref{eq:barH_ce_NC1} and~\eqref{eq:barH_ce} correspond to ${\cal NC}_1$ and ${\cal NC}_{2-3}$ of the UPFM under cross-entropy loss.
\item The global minimizer to~\eqref{def:upfm_l2} is given by
\begin{equation}\label{eq:barH_l2_NC1}
\begin{aligned}
& \BH^{\top}\BH = \sqrt{\frac{\lambda_W}{\lambda_H}}\frac{(1- \sqrt{n}K\lambda)_+}{\sqrt{n}} \BY^{\top}\BY = \sqrt{\frac{\lambda_W}{\lambda_H}}\frac{(1- \sqrt{n}K\lambda)_+}{\sqrt{n}}\I_K\otimes\BJ_n, \\
& \BW^{\top}\BW =\sqrt{\frac{\lambda_H}{\lambda_W}}n^{1/2} (1 - \sqrt{n}K\lambda)_+ \I_K,~~~ \BW^{\top}\BH = (1-\sqrt{n}K\lambda)_+ \BY
\end{aligned}
\end{equation}
i.e., the mean feature $\bar{\BH}\in\RR^{D\times K}$ is nonnegative and satisfies 
\begin{equation}\label{eq:barH_l2}
\bar{\BH}^{\top}\bar{\BH} = \sqrt{\frac{\lambda_W}{\lambda_H}}\frac{(1- \sqrt{n}K\lambda)_+}{\sqrt{n}}\I_K,~~~\BW^{\top}\bar{\BH} = (1-\sqrt{n}K\lambda)_+ \I_K,~~~\BW = \sqrt{\frac{n\lambda_{H}}{\lambda_{W}}}\bar{\BH}
\end{equation}
where $\lambda = \sqrt{\lambda_W\lambda_H}$ and $\BY$ is defined in~\eqref{def:Y}. Here~\eqref{eq:barH_l2_NC1} and~\eqref{eq:barH_l2} correspond to ${\cal NC}_1$ and ${\cal NC}_{2-3}$ of the UPFM under $\ell_2$-loss.
\end{enumerate}
\end{theorem}
The same result is obtained in~\cite{DTNH24} for the unconstrained positive feature model under imbalanced datasets. 
The proof of Theorem~\ref{thm:upfm} is provided in Section~\ref{ss:upfm}, and our technique uses a convex relaxation of~\eqref{def:upfm_ce} and~\eqref{def:upfm_l2}, and show that the collapsed solution is exactly the global minimizer to the convex relaxation. Our justification of Theorem~\ref{thm:upfm} is significantly different from~\cite{DTNH24} and can also be extended to the imbalanced scenarios, and thus we provide a proof here.
From Theorem~\ref{thm:upfm}, we can directly see that the global minimizer shows the within-class variability collapse and the convergence of mean features to an orthogonal frame under the unconstrained positive  feature model.

Once we have a full characterization of the global minimizer to~\eqref{def:upfm_ce} and~\eqref{def:upfm_l2}, we will study whether the global minimizer exhibits the ${\cal NC}$ property for a given dataset $\{(\bx_i,\by_i)\}_{i=1}^N$ and $N=Kn$. We begin with a two-layer neural network with ReLU activation, and the corresponding feature is exactly
\[
\BH = \sigma_{\ReLU}(\BW_1^{\top}\BX)\in\RR^{D\times N}
\]
where $\BW_1\in\RR^{d\times D}$, and $d$ and $D$ are the dimensions of the input data and output features respectively.
Assume $\BX = [\BX_1,\cdots,\BX_K]\in\RR^{d\times Kn}$ with $\BX_{k} = [\bx_{k 1},\cdots,\bx_{k n}]\in\RR^{d\times n}$, i.e., $\BX_k$ is the $k$-th class consisting of $n$ points in $\RR^d$.

To see if the neural collapse occurs, Theorem~\ref{thm:upfm} implies that we need to find out the existence of $\BW_1$ such that $\BH = \sigma_{\ReLU}(\BW_1^{\top}\BX) = \bar{\BH} \BY$ with $\bar{\BH}\geq 0$ and $\bar{\BH}^{\top}\bar{\BH}\propto \I_K$. 
Note that for any nonnegative mean feature matrix $\bar{\BH}$ with orthogonal columns, it satisfies
\[
\supp(\bar{\bh}_k)\cap \supp(\bar{\bh}_{\ell}) = \emptyset.
\]
where $\bar{\bh}_k\geq 0$ and $\|\bar{\bh}_k\|$ is constant over $1\leq k\leq K$. In other words, each row of $\bar{\BH}\in\RR^{D\times K}$ is a one-hot vector in $\RR^K$ multiplied by a nonnegative scaler. 
This reduces to
show the existence of a vector $\bbeta_{k}\in\RR^d$ such that
\begin{equation}\label{featc:event}
\BX_{k}^{\top}\bbeta_{k} = \bone_n,~~~~\BX^{\top}_{k'} \bbeta_{k} \leq 0,~~\forall 1\leq k'\neq k \leq K,
\end{equation}
for each $1\leq k\leq K$, which is a linear feasibility problem.
The answer depends on $K$, $d$ and $n$. 

\vskip0.25cm

Our first result is quite general and applies to any dataset, and it concerns whether the global minimizer to~\eqref{def:dnn2} contains the ${\cal NC}$ configuration, i.e., whether there exists $\BW_1$ such that $\BH = \sigma(\BW_1^{\top}\BX) = \bar{\BH} \otimes \bone_n^\top$ for some $\bar{\BH}\in\RR_+^{D\times K}$.

\begin{theorem}[\bf Neural collapse for general datasets]\label{thm:NC_general}
For any general dataset $\BX$, we have the following results:
\begin{enumerate}[(a)]
\item Suppose $\BX$ is not linearly separable, then the neural collapse does not occur.
\item Suppose $d\geq Kn$ and moreover $\BX^{\top}\BX$ is of rank $Kn$, and then~\eqref{featc:event} satisfies for any $1\leq k\leq K$.
\item Suppose $d < n$ and $\BX_k^{\top}$ does not contain $\bone_n$ in its range, then the neural collapse will not occur. 
\end{enumerate}
\end{theorem}

The proof is very simple and thus we present it here. 
The interesting implication of Theorem~\ref{thm:NC_general} is that even if the neural network is extremely wide, ${\cal NC}$ may not happen if the input dimension $d$ is too small, which answers Question 1.

\begin{proof}[\bf Proof of Theorem~\ref{thm:NC_general}]
For (a), we prove it by contradiction. Suppose the neural collapse occurs, then 
\[
\BX_{k}^{\top}\bbeta_{k} = \bone_n,~~~~\BX^{\top}_{k'} \bbeta_{k} \leq 0,~~\forall k'\neq k.
\]
It means the $k$-th cluster $\BX_k$ is separated from the rest by the hyperplane $\{\bx: \lag \bx, \bbeta_k\rag =1/2\}$ for all $k$. Therefore, if the linear separability fails to hold, then the neural collapse will not occur.
For (b), suppose $\BX$ is of rank $Kn$, then the range of $\BX$ is of dimension $Kn$ and definitely contains a vector such that~\eqref{featc:event} holds.
For (c), suppose $\BX_k^{\top}$ does not $\bone_n$ in its column space, then the first equality in~\eqref{featc:event} cannot hold.
\end{proof}

\subsection{${\cal NC}$ and SNR}
Therefore, the more interesting scenario is $n < d\leq Kn$ as when the input data dimension is greater than the data size, i.e., $d> Kn$, it is very likely that the ${\cal NC}$ will occur due to the linear feasibility in~\eqref{featc:event}. 
To have a concrete discussion on the ${\cal NC}$ with $d< Kn$, 
 we consider the Gaussian mixture model with $K$ classes:
\[
\BX_{k} = \bmu_{k}\bone_n^{\top} + \sigma\BZ_{k}\in\RR^{d\times n}.
\]
We start with a two-layer neural network with a Gaussian mixture model (GMM) of two clusters and then extend to the $K$-class scenario. 
The goal is to see how the signal-to-noise ratio (SNR) in the data affects the ${\cal NC}$.
The data are assumed in the following form:
\[
\BX^{\top} := \begin{bmatrix}\BX_1^{\top} \\ \BX_2^{\top} \end{bmatrix}:= 
\begin{bmatrix}
 \bone_n\bmu_1^{\top} + \sigma\BZ_1 \\
\bone_n\bmu_2^{\top}+ \sigma\BZ_2
\end{bmatrix}\in\RR^{2n\times d}
\]
where $\bmu_1$ and $\bmu_2$ are the cluster centers, and each cluster consists of $n$ samples in $\RR^d.$

\begin{theorem}[\bf Neural collapse for GMM with two clusters]\label{thm:nc_gmm2}
 Let $0<\eps<1$, $2n \geq d > n+C\eps^{-2}\log{n}$, and $\theta$ be the angle between $\bmu_1$ and $\bmu_2$. With probability at least $1-O(n^{-1})$,~\eqref{featc:event} is feasible, i.e., the global minimizer to the empirical risk is given by the ${\cal NC}$ configuration under either scenario below:
 \begin{itemize}
\item For $\cos\theta< -4\eps/(1-\eps)^2$, i.e., $\theta > \pi - \arccos (4\eps/(1-\eps)^2)$, then
\begin{equation}\label{thmeq:gmm2_eq_a}
\sigma \lesssim (1-\eps)\sqrt{\frac{d-n}{d\log n}}\min\{ \|\bmu_1\|,\|\bmu_2\|\}.
\end{equation}
\item For $\cos\theta> -4\eps/(1-\eps)^2$, i.e., $\theta < \pi - \arccos (4\eps/(1-\eps)^2)$, then
\begin{equation}\label{thmeq:gmm2_eq_b}
\sigma \lesssim (1-\eps) \sqrt{\frac{d-n}{d\log n} \left( 1- \left( |\cos\theta| + \frac{4\eps}{(1-\eps)^2}\right)^2\right) } \min\{\|\bmu_1\|,\|\bmu_2\|\}.
\end{equation}
\end{itemize}
In particular, if $\bmu_1 = -\bmu_2 = \bmu$, i.e., $\theta=\pi$, then
\[
\frac{\sigma}{\|\bmu\|} \lesssim (1-\eps) \sqrt{\frac{d-n}{d\log n}},
\]
then neural collapse occurs with probability at least $1 - O(n^{-1}).$
Suppose $d \geq 2n$,~\eqref{featc:event} is feasible with probability 1, i.e., the neural collapse occurs.
\end{theorem}
In other words, if $\sigma$ is sufficiently small, then even if $d < 2n$, the neural collapse still occurs. 
Now we proceed to consider data that satisfy the GMM with $K$ classes, i.e.
\begin{equation}\label{eq:gmmk}
\BX^{\top} := \begin{bmatrix}\BX_1^{\top} \\ \vdots \\ \BX_K^{\top} \end{bmatrix}:= 
\begin{bmatrix}
 \bone_n\bmu_1^{\top} + \sigma\BZ_1 \\
\vdots \\
\bone_n\bmu_K^{\top}+ \sigma\BZ_K
\end{bmatrix}\in\RR^{Kn\times d}
\end{equation}
where $\{\BZ_k\}_{k=1}^K$ are $K$ i.i.d. $n\times d$ Gaussian random matrices. 

We define the mean vector matrix by
\[
\BPi = [\bmu_1,\bmu_2,\cdots,\bmu_K]^{\top}\in\RR^{K\times d}.
\]
Also, we assume $\BPi$ is a matrix of full row rank. If it is not full rank, we simply pick the maximum independent set of $\{\bmu\}_{k=1}^K$, and then the theorem above with $K=2$ generalizes to the $K$-class case. 
\begin{theorem}[\bf Neural collapse for GMM with $K$ clusters]\label{thm:nc_gmmk}
For the GMM with $K$ clusters, the following holds.
\begin{enumerate}[(a)]
\item Suppose $\BPi$ is full rank, $d-n\geq CK^2\log n$, and
\begin{equation}\label{thmeq:gmmk_eq_a}
\sigma \lesssim \sqrt{\frac{d-n}{d \log (Kn)}}\cdot\frac{\sigma_{\min}(\BPi)}{\sqrt{K-1}},
\end{equation}
then there exists a global minimizer to the empirical risk minimizer which is given by the ${\cal NC}$ configuration with probability at least $1-O(n^{-1}).$

\item Suppose 
\begin{equation}\label{thmeq:gmmk_eq_b}
\frac{d}{n} \geq \frac{K+1}{2} + 2\sqrt{\frac{(K-1)\log n}{n}} +\frac{K+2\log n}{n},
\end{equation}
then the neural collapse occurs with a high probability of least $1-O(n^{-1})$.

\item Suppose $ d\geq Kn$, then the neural collapse occurs with probability one since $\BX$ is of rank $Kn.$
\end{enumerate}
\end{theorem}
Theorem~\ref{thm:nc_gmm2} and~\ref{thm:nc_gmmk}(a) provide an answer to Question 2. If the data has a cluster structure, i.e., $\sigma$ is small compared with $\|\bmu\|$, the ${\cal NC}$ is more likely to occur. The proof follows from constructing a vector that satisfies~\eqref{featc:event} with high probability by using the union bound. Theorem~\ref{thm:nc_gmmk}(b) implies that for the GMM, $d\geq Kn$ is not necessary to guarantee the emergence of the ${\cal NC}$ and in fact $d\geq (K+1)n/2$ suffices. The main technique is to first reformulate the feasibility of~\eqref{featc:event} as a problem of finding the maximum of a Gaussian process. Then we apply Gordon's bound to find a sharp upper bound of the Gaussian process and obtain Theorem~\ref{thm:nc_gmmk}(b). The proofs of Theorem~\ref{thm:nc_gmm2} and~\ref{thm:nc_gmmk} are provided in Section~\ref{ss:2nn_fp}.


Note that Theorem~\ref{thm:NC_general} provides a negative answer to Question 1 for two-layer neural networks. However, things become interesting if we turn to three-layer neural networks with the feature map matrix equal to
\begin{equation}\label{eq:three_layer_feature}
\BH(\BX) = \sigma_{\ReLU}\left(\BW_2^\top\sigma_{\ReLU}\left(\BW_1^\top\BX\right)\right)
\end{equation}
where $\BW_1\in \mathbb{R}^{d \times d_1}$ and $\BW_2\in \mathbb{R}^{d_1 \times D}$, and $\BX \in \mathbb{R}^{d \times N}$ represents a dataset of $N$ points. 
The question of whether there exists $\BW_1$ and $\BW_2$ such that $\BH(\BX) = \bar{\BH} \otimes \bone_n^\top$ can be simplified to a two-layer scenario: it suffices to consider a random feature model: each $\BW_1\in\RR^{d\times d_1}$ is assumed to be i.i.d. $\mathcal{N}(0,1/d_1).$ In other words, for a given input $\bx\in\RR^d$, the output through the first layer $\sigma_{\ReLU}(\BW_1^{\top}\bx)\in\RR^{d_1}$ exactly satisfies truncated normal distribution.

Based on Theorem~\ref{thm:NC_general}, to induce the ${\cal NC}$, it suffices to ensure that $\sigma_{\ReLU}(\BW_1^{\top}\BX)$ is of rank $N = Kn$ for a sufficiently large $d_1$, which is guaranteed by the following theorem.
\begin{theorem}[\bf Neural collapse for three layer random feature network]\label{Thm:three_layer_main}
Suppose $\{\bx_1,\cdots,\bx_N\}$ is a given dataset with any pair of them non-parallel and $\|\bx_i\| = 1.$
Given a three-layer network with feature map~\eqref{eq:three_layer_feature} and $\BW_1$ is i.i.d. $\mathcal{N}(0,1/d_1)$, if we have
\[
d_1 \gtrsim \frac{\|\BX\|^4}{\lambda^2_{\min}(\BH_{\infty})}\cdot N \log N
\]
where $\BH_{\infty}$ is the kernel matrix that equals:
\[
[\BH_{\infty}]_{ij} = \E_{\bz\sim\mathcal{N}(0,\I_d)} \left( \sigma_{\ReLU} (\lag \bz, \bx_i\rag) - \sqrt{\frac{2}{\pi}}\right)\left(\sigma_{\ReLU}(\lag \bz, \bx_j\rag) - \sqrt{\frac{2}{\pi}}\right).
\]
Then, with probability at least $1-O(N^{-2})$, the nonnegative matrix $\sigma_{\ReLU}\left(\BW_1^\top\BX\right)\in\RR^{d_1\times N}$ is of full column rank. Hence, Theorem~\eqref{thm:NC_general}(b) implies that neural collapse occurs.
\end{theorem}
 In other words, Question 1 has a positive answer (the $\mathcal{NC}$ will occur) as long as the neural network has a depth of more than 2 and the first layer is sufficiently wide with width $d_1\gtrsim N\log N$. The proof of Theorem~\eqref{Thm:three_layer_main} consists of two main steps. First, we prove that $\BH_{\infty}$ is positive definite, i.e., $\lambda_{\min}\left(\BH_{\infty}\right) > 0$. Then, we apply concentration inequalities to control the spectral deviation between $d_1^{-1} \left(\sigma_{\ReLU}\left(\BW_1^\top\BX\right)\right)^{\top} \sigma_{\ReLU}\left(\BW_1^\top\BX\right)$ and $\BH_{\infty}$ to prove $\sigma_{\ReLU}\left(\BW_1^\top\BX\right)$ is full rank. The details are presented in Section~\ref{ss:3layer}.

\vskip0.25cm

\subsection{${\cal NC}$ and generalization}
This section is devoted to Question 3: if ${\cal NC}$ occurs, does it necessarily imply good generalization?
To understand this question, we consider the task of binary classification for a two-layer ReLU network. In particular, throughout the discussion in this section, we will focus on the best misclassification error a simplified two-neuron classifier can achieve in the presence of the ${\cal NC}.$ We need to make some preparations before proceeding to our main results.

\paragraph{Data model:} For the data generative model, we assume they are sampled from GMM with two clusters and the mean vectors are opposite:
\begin{equation}\label{eq:gmm2_rade}
\bx = \xi \bmu + \sigma\bz,~~~\bz\sim\mathcal{N}(0,\I_d)
\end{equation}
where $\xi$ is a Rademacher random variable
and the training data are
\begin{equation}\label{gen:data}
\BX^{\top} = \begin{bmatrix} \bone_n \bmu^\top \\ -\bone_n \bmu^\top \end{bmatrix} + \sigma \BZ\in\RR^{2n\times d}
\end{equation}
where $\pm\bmu \in \mathbb{R}^d$ denotes the class mean and $\BZ \in \mathbb{R}^{2n \times d}$ is a Gaussian random matrix. 

\paragraph{Model under the ${\cal NC}$:}
We now consider the generalization ability for a trained model that exhibits neural collapse:
\begin{equation}\label{eq:gen_model}
f_{\theta}(\bx) = \BW^\top\sigma_{\ReLU}(\BW_1^\top\bx)
\end{equation}
where $\BW_1 \in \mathbb{R}^{d \times D}$ and $\BW \in \mathbb{R}^{D \times 2}$. We denote the $i$-th column of $\BW_1$ by $\bbeta_i$, $1\leq i\leq D$ and the $k$-th column of $\BW$ by $\balpha^{k},~1\leq k\leq 2$. We first try to rewrite~\eqref{eq:gen_model} by exploiting the information of ${\cal NC}$.
 Under the feature collapse ${\cal NC}_1$, we have
\[
\sigma_{\ReLU}(\BW_1^\top \bx_{ki}) = \bar{\bh}_k,~~~\forall~1\leq k\leq 2,~1\leq i\leq n,
\]
where $\bx_{ki}$ is the $i$-th sample in the $k$-th class and $\bar{\bh}_k \in \mathbb{R}^{D}_{+}$ is the mean feature for the class $k$. Based on Theorem~\ref{thm:upfm}(a), we have $\lag \bar{\bh}_1,\bar{\bh}_2\rag=0$ and $\|\bar{\bh}_1\| = \|\bar{\bh}_2\|,$ and thus
\begin{equation}\label{eq:cm_weight}
\BW = [\balpha, -\balpha],~~~\balpha:=\balpha^1 = -\balpha^2 \propto \bar{\bh}_1 -\bar{\bh}_2, ~~~\supp(\bar{\bh}_1) \cap \supp(\bar{\bh}_2) = \emptyset
\end{equation}
since $\BW^{\top}[\bar{\bh}_1,\bar{\bh}_2 ] \propto \I_2 - \BJ_2/2$ and $\BW\propto [\bar{\bh}_1,\bar{\bh}_2](\I_2 - \BJ_2/2) = [\bar{\bh}_1-\bar{\bh}_2, \bar{\bh}_2-\bar{\bh}_1]\in\RR^{D\times 2}$.

We denote $S_1 = \{i : \bh_{1i} \geq 0, \bh_{2i} = 0,~1\leq i\leq D \}$ and $S_2 =\{1,\cdots,D\} \setminus S_1$, and then the classifier becomes
\begin{equation}\label{eq:ncgen_ob}
\begin{aligned}
f_{\theta}(\bx)&=\sum_{i\in S_1} \alpha_{i}\sigma_{\ReLU}(\lag \bbeta_i, \bx\rag) - \sum_{i\in S_2} \alpha_{i}\sigma_{\ReLU}(\lag \bbeta_i, \bx\rag) 
\end{aligned}
\end{equation}
where $\alpha_i > 0$ for $i\in S_1$ and $\alpha_i < 0$ otherwise.
Suppose $\bx$ is sampled from the first class, i.e., $\bx = \bmu + \sigma\bz$, then $f_{\theta}(\bx)$ produces a correct classification if $f_{\theta}(\bmu + \sigma\bz) > 0$ and thus to upper bound the misclassification error, we need to control $\Pr(f_{\btheta}(\mu + \sigma\bz) < 0)$.
This quantity is too complicated to compute exactly because the actual $\bbeta_i$ is unknown and $f_{\btheta}(\cdot)$ involves the sum of $D$ terms. 

\paragraph{\bf Simplification to a two-neuron classifier:} We simplify it by considering the performance of a two-neuron classifier:
\begin{equation}\label{eq:f_reduce}
f(\bx) = \sigma_{\ReLU}(\bbeta_1^{\top}\bx) - \sigma_{\ReLU}(\bbeta_2^{\top}\bx)
\end{equation}
and study whether it is able to correctly classify a given data point. 
This simplification follows from two observations: (a) this two-neuron classifier is a reduced form of~\eqref{eq:ncgen_ob} by setting both $|S_1| = |S_2| = 1$; 
(b) to achieve a good generalization performance on binary classification on the data sampled from GMM, it suffices to find one single hyperplane that is able to separate the two classes. Hence, understanding the generalization of this simple model sheds some light on the general $D$-neuron case.


Our analysis of the two-neuron classifier  relies on our understanding on the output of a single neuron: $g_{\bbeta}(\bx) := \sigma_{\ReLU}(\bbeta^{\top}\bx)$ and consider
\begin{equation}\label{eq:gnc}
{\cal C}_1 : = \{\bbeta: \BX_1^{\top} \bbeta = \bone_n,~~\BX_2^{\top}\bbeta \leq 0\},
~~~~{\cal C}_2 = \{\bbeta: \BX_2^{\top} \bbeta = \bone_n,~~\BX_1^{\top}\bbeta \leq 0\},
\end{equation}
i.e., for any $\bbeta\in{\cal C}_1$, $g_{\bbeta}(\bx)$ maps the first class of training data to 1 and the second class to 0 which means the feature is collapsed. A similar counterpart holds for $\bbeta\in{\cal C}_2.$ In other words, for $\bbeta_k\in {\cal C}_k$, then $f(\bx)$ in~\eqref{eq:f_reduce} achieves the feature variability collapse.

To estimate the misclassification of~\eqref{eq:f_reduce},
we start with the misclassification of $g_{\bbeta}(\bx)$ for $\bbeta\in{\cal C}_k$.
First note that for any $\bbeta\in {\cal C}_1$, it holds that
\begin{align*}
\frac{1}{2}\Big[ \Pr(g_{\bbeta_1}(\bmu+\sigma\bz) \leq 0) + \Pr(g_{\bbeta_1}(-\bmu+\sigma\bz) > 0)\Big] & = \frac{1}{2}\Big[ \Pr( \lag \bmu+\sigma\bz, \bbeta_1\rag \leq 0) + \Pr(\lag -\bmu + \sigma\bz, \bbeta_1\rag > 0)\Big] \\
& = \frac{1}{2}\Pr_{z\sim\mathcal{N}(0,1)}\left( |z| \geq \frac{\lag \bmu,\bbeta_1\rag}{\sigma\|\bbeta_1\|} \right) = \Phi\left( -\frac{\lag \bmu,\bbeta_1\rag}{\sigma\|\bbeta_1\|} \right) 
\end{align*}
where $\Phi(\cdot)$ is the c.d.f. of standard normal distribution. 
Similarly for $\bbeta_2\in {\cal C}_2$, then
\[
\frac{1}{2}\Big[ \Pr(g_{\bbeta_2}(\bmu+\sigma\bz) > 0) + \Pr(g_{\bbeta_2}(-\bmu+\sigma\bz) \leq 0)\Big] = \frac{1}{2}\Pr_{z\sim\mathcal{N}(0,1)}\left( |z| \geq - \frac{\lag \bmu,\bbeta_2\rag}{\sigma\|\bbeta_2\|} \right) = \Phi\left( \frac{\lag \bmu,\bbeta_2\rag}{\sigma\|\bbeta_2\|} \right).
\]

Therefore, under the ${\cal NC}$, the best generalization and misclassification error of~\eqref{eq:f_reduce} is closely related to the global maximum of the following problems:
\begin{equation}\label{eq:nc_gen_opt}
\max_{\bbeta\in{\cal C}_1}~\frac{\lag \bmu, \bbeta\rag}{\|\bbeta\|}~~~\text{ and }~~~
\max_{\bbeta\in{\cal C}_2}~-\frac{\lag \bmu, \bbeta\rag}{\|\bbeta\|}.
\end{equation}
The following theorems provide an estimation of~\eqref{eq:nc_gen_opt} in two different regimes.

\begin{theorem}\label{thm:nc_gen}
Consider the data sampled from~\eqref{gen:data} and 
the following statements hold true:
\begin{enumerate}[(a)]
\item Suppose $d-n+1 > C\eps^{-2}\log n$ and 
  \begin{equation}\label{thmeq:nc_gen_a}
 \frac{\sigma}{\|\bmu\|} \lesssim (1-\eps)\sqrt{\frac{d-n+1}{d \log{n}}}, 
  \end{equation}
 then with probability at least $1-O(n^{-1})$ there exists $\bbeta\in\RR^{d\times D}$ that can induce~\eqref{eq:gnc}. Moreover, we have
\begin{align*}
\min_{\bbeta\in{\cal C}_1} \Phi\left( -\frac{\lag \bmu,\bbeta\rag}{\sigma\|\bbeta\|} \right)\lesssim n^{-2},~~~\min_{\bbeta\in{\cal C}_2} \Phi\left( \frac{\lag \bmu,\bbeta\rag}{\sigma\|\bbeta\|} \right) \lesssim n^{-2}.
\end{align*}

\item For $d \geq 2n \log n$, then the following holds
\begin{equation}\label{eq:error_gen}
\begin{aligned}
& \min_{\bbeta\in{\cal C}_1} \Phi\left( -\frac{\lag \bmu,\bbeta\rag}{\sigma\|\bbeta\|} \right),~~\min_{\bbeta\in{\cal C}_2} \Phi\left( \frac{\lag \bmu,\bbeta\rag}{\sigma\|\bbeta\|} \right) \\ 
& ~~\geq 1- \Phi\left( \left(
 \frac{n }{2d}\left(\frac{s^3 e^{-\frac{1}{2s^2} }}{\sqrt{2\pi}} + s^2+1-(c_1s^2+c_2s)\sqrt{\frac{\log n}{n}}\right) + s^2\right)^{-1/2} \right),
\end{aligned}
\end{equation}
where $s = \sigma/\|\bmu\|$, and $c_1$ and $c_2>0$. 
\end{enumerate}
\end{theorem}
Theorem~\ref{thm:nc_gen}(a) focuses on the estimation of~\eqref{eq:nc_gen_opt} in the low-noise regime while Theorem~\ref{thm:nc_gen}(b) concerns the regime when $d$ is sufficiently large and the ${\cal NC}$ always occurs. 
Here we briefly describe the key steps of the proof. The key idea is to approximate the global maximum of~\eqref{eq:nc_gen_opt}. However, it is not straightforward to estimate it exactly. Therefore, we relax~\eqref{eq:nc_gen_opt} by dropping the affine inequality constraints and solving it. Then we consider the sufficient conditions under which the global maximizer to the relaxed problem is also the global maximizer to~\eqref{eq:nc_gen_opt}. This leads to the proof and conclusion of Theorem~\ref{thm:nc_gen}.
For Theorem~\ref{thm:nc_gen}(b) with $d\geq 2n$, it is different from Theorem~\ref{thm:nc_gen}(a) because Theorem~\ref{thm:NC_general} implies that the ${\cal NC}_1$ occurs with probability 1. The feasibility of the ${\cal NC}$ makes it possible to approximate the global maximizer to~\eqref{eq:nc_gen_opt} directly by rewriting the affine constraints in ${\cal C}_1$ and ${\cal C}_2$. The technical details are deferred to Section~\ref{ss:twonn_gen}. With the theorem above, we can obtain a characterization of the generalization performance of a two-neuron classifier in presence of the ${\cal NC}.$

\begin{theorem}[\bf Misclassification of a two-neuron classifier under the ${\cal NC}$]\label{thm:2neuron}
Consider $f(\bx) = g_{\bbeta_1}(\bx) - g_{\bbeta_2}(\bx)$ in~\eqref{eq:f_reduce} where $\bbeta_k\in{\cal C}_k$ for $k=1,2$.
\begin{enumerate}[(a)]
\item Under the assumption of Theorem~\ref{thm:nc_gen}(a), there exists $\bbeta_k\in{\cal C}_k$ such that the misclassification error of $f(\bx)$ in~\eqref{eq:f_reduce} is bounded by $O(n^{-2})$;
\item Under the assumption of Theorem~\ref{thm:nc_gen}(b), the misclassification error of $f(\bx)$ in~\eqref{eq:f_reduce} for any $\bbeta_k\in{\cal C}_k$ is at least 
\[
1 - \Phi\left( \left(
 \frac{n }{2d}\left(\frac{s^3 e^{-\frac{1}{2s^2} }}{\sqrt{2\pi}} + s^2+1-(c_1s^2+c_2s)\sqrt{\frac{\log n}{n}}\right) + s^2\right)^{-1/2} \right)
\]
where $s = \sigma/\|\bmu\|$ measures the noise level in~\eqref{eq:gmm2_rade}.
\end{enumerate}
\end{theorem}
The theorem above implies that in the low-noise regime and $d-n+1\gtrsim \log n$, there exists a two-neuron classifier that achieves the ${\cal NC}$ and also enjoys an excellent generalization bound. For $d\geq 2n\log n$ and also the ${\cal NC}$ occurs, any two-neuron classifier that achieves the ${\cal NC}$ does not have a small misclassification error if the noise level $\sigma/\|\bmu\|$ is large. This provides a partial answer to Question 3.
\begin{proof}[\bf Proof of Theorem~\ref{thm:2neuron}]
The misclassification error of a two-neuron classifier in~\eqref{eq:f_reduce} is given by
\[
\frac{1}{2} \Big[ \Pr( g_{\bbeta_1}(\bmu+\sigma\bz) < g_{\bbeta_2}(\bmu + \sigma\bz) ) + \Pr( g_{\bbeta_1}(-\bmu+\sigma\bz) > g_{\bbeta_2}(-\bmu + \sigma\bz) )\Big].
\]

An upper bound of the error is given by
\begin{align*}
& \frac{1}{2} \Big[ \Pr( g_{\bbeta_2}(\bmu+\sigma\bz) > 0 ) + \Pr( g_{\bbeta_1}(-\bmu + \sigma\bz) > 0)\Big] \\
& = \frac{1}{2} \Big[ \Pr(\sigma \lag \bbeta_1, \bz\rag > \lag \bmu,\bbeta_1\rag) + \Pr(\sigma\lag \bbeta_2,\bz\rag \geq -\lag \bmu, \bbeta_2\rag) \Big] = \frac{1}{2} \Big[ \Phi\left(- \frac{\lag \bmu,\bbeta_1\rag}{\|\bbeta_1\|} \right) + \Phi\left(\frac{\lag \bmu,\bbeta_2\rag}{\|\bbeta_2\|} \right) \Big] \lesssim n^{-2}.
\end{align*}
Theorem~\ref{thm:nc_gen}(a) implies in the low noise regime, there exists $\bbeta_k\in{\cal C}_k$ such that the corresponding two-neuron classifier enjoys a very small misclassification error. 

On the other hand, a lower bound is given by
\begin{align*}
& \frac{1}{2} \Big[ \Pr( g_{\bbeta_1}(\bmu+\sigma\bz) < g_{\bbeta_2}(\bmu + \sigma\bz) ) 
+ \Pr( g_{\bbeta_1}(-\bmu+\sigma\bz) > g_{\bbeta_2}(-\bmu + \sigma\bz) )\Big] \\
& \geq \frac{1}{2} \Big[ \Pr( g_{\bbeta_1}(\bmu+\sigma\bz) \leq 0 ) 
+ \Pr( g_{\bbeta_2}(-\bmu + \sigma\bz) ) \leq 0)\Big] \\
& = \frac{1}{2} \Big[ \Pr( \lag \bmu+\sigma\bz, \bbeta_1\rag \leq 0 ) 
+ \Pr( \lag -\bmu+\sigma\bz, \bbeta_2\rag \leq 0)\Big] = \frac{1}{2}\left[\Phi\left( -\frac{\lag \bmu,\bbeta_1\rag}{\sigma\|\bbeta_1\|}\right) + \Phi\left( \frac{\lag \bmu,\bbeta_2\rag}{\sigma\|\bbeta_2\|} \right)\right].
\end{align*}
Theorem~\ref{thm:nc_gen}(b) implies that as noise level $\sigma$ increases, even if $d$ is large and ${\cal NC}$ occurs (recall that if $d > 2n,$ the ${\cal NC}$ occurs), any $\BW_1$ that induces ${\cal NC}$ will not achieve a misclassification error smaller than~\eqref{eq:error_gen}. 
For example, if $s$ is large, i.e., $s \gtrsim \sqrt{n/d}$, then~\eqref{eq:error_gen} is roughly approximated by
\[
1 - \Phi\left( \sqrt{ \frac{d}{n }} \frac{1}{s^{3/2}} \right),
\]
which is close to $1/2$. For any $\bbeta_k\in {\cal C}_k$, it holds that the misclassification error is at least $1 -\Phi\left( \sqrt{ \frac{d}{n }} \frac{1}{s^{3/2}} \right)$.
\end{proof}

\section{Experiments 
}\label{s:numerics}
\subsection{${\cal NC}_1$ for two-layer neural networks}
This subsection aims to verify the sufficient conditions for inducing the $\mathcal{NC}$ derived in Theorem~\ref{thm:nc_gmm2} and~\ref{thm:nc_gmmk}. For both theorems, we verify our bounds by sampling data from the GMM in~\eqref{eq:gmmk}.
We solve the following linear programming to examine the feasibility of the ${\cal NC}$:
\begin{equation}\label{numer_fea1}
\text{Find}~\bbeta~\text{such that}~ \left(\bone_n\bmu_1^\top + \sigma \BZ_1\right)\bbeta = \bone_n, ~~\left(\bone_n\bmu_k^\top + \sigma \BZ_k\right)\bbeta \leq 0,~~(k \neq 1).
\end{equation}
Note that here we only consider the linear feasibility problem regarding the first class because a similar conclusion can be drawn for other classes. For this test, the goal is to study the feasibility and its dependence on $d,n,K$ and $\sigma$. 

For $K=2$, we let $n = 300$ and $d/n$ vary from 1 to 2. To simplify our discussion, we consider the mean vectors having the same norm $\|\bmu_1\| = \|\bmu_2 \| = \|\bmu\|=1$ and denote their angle by $\theta$.
For each pair of $(d,\sigma)$, we run $10$ experiments and solve~\eqref{numer_fea1}, and if the linear program is feasible, then we count it as a successful instance and we plot the successful rate of each set of parameters. From Figure~\ref{fig:collapse_tc}, we can see that Theorem~\ref{thm:nc_gmm2} and~\ref{thm:nc_gmmk} does not exactly match the phase transition but provides a good approximation. For lower $\sigma$ and larger $\theta$, then white regions are clearly larger, showing that well-clustered data are more likely to induce the ${\cal NC}$. For the case when the signal-to-noise ratio of the GMM model is low, e.g., for $\sigma$ is large or for $\theta=0$, we can see $d/n=3/2 = (K+1)/2$ provides an accurate description of the phase transition, which is guaranteed by the Gordon's bound.

\begin{figure}[h!]
  \centering
   \includegraphics[width=160mm]{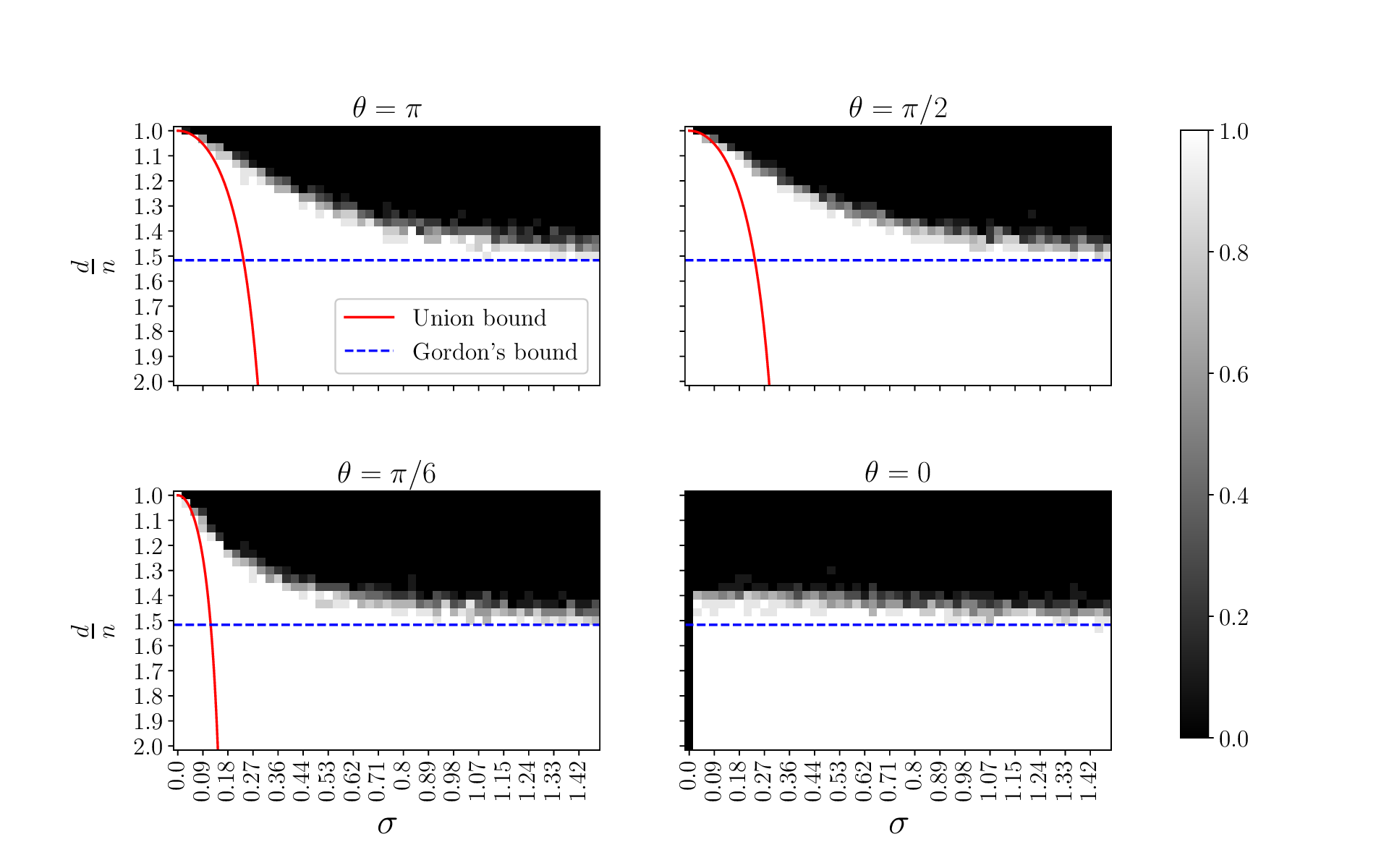}
  \caption{The ${\cal NC}$ feasibility plot for a two-class GMM with mean vectors $\bmu_1$ and $\bmu_2$, and $\theta$ is their angle. The legends denote the sufficient conditions provided by Theorem~\ref{thm:nc_gmm2} and~\ref{thm:nc_gmmk} represented respectively by red solid line (union bound) and blue dashed line (Gordon's bound). 
}
  \label{fig:collapse_tc}
\end{figure}

\begin{figure}[h]
  \centering
   \includegraphics[width=130mm]{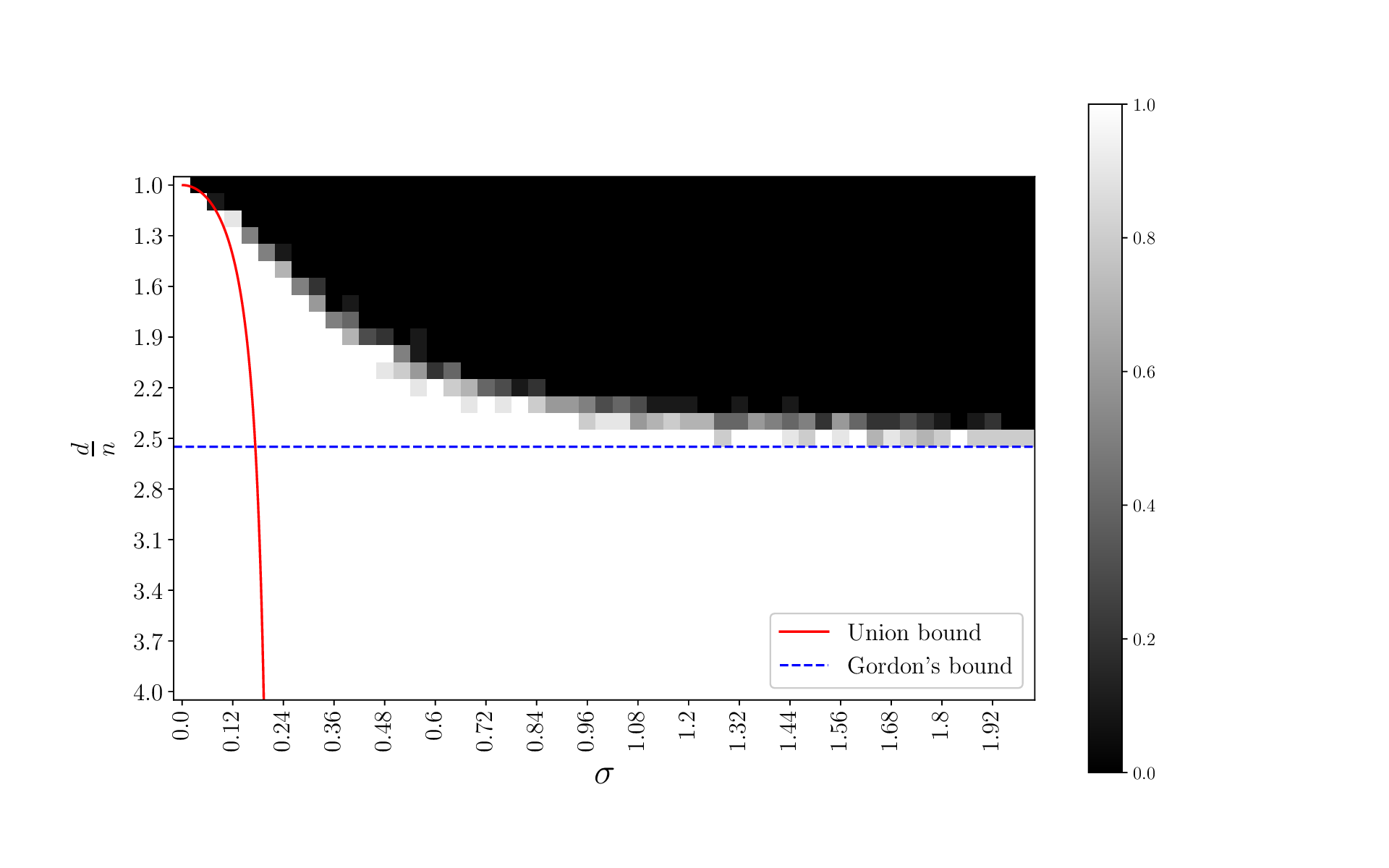}
  \caption{${\cal NC}_1$ feasibility plot for GMM with four clusters and $\bmu_k = \be_k$,~$1\leq k\leq 4$.}
  \label{fig:collapse_mc}
\end{figure}

For $K\geq 2$, the setup of experiments is similar to the $K=2$ case: $K=4,n = 250$, $\bmu_k = \be_k$, and $d/n$ ranges from 1 to 4. For each $(d,\sigma)$, 10 experiments are carried out. Figure~\ref{fig:collapse_mc} shows a similar phase transition plot as in Figure~\ref{fig:collapse_tc}. In the high and low $\sigma$ regime, Gordon's bound (blue dashed line) and union bound (red solid line) approximate the phase transition boundary respectively. Overall, our characterization given by Theorem~\ref{thm:nc_gmm2} and~\ref{thm:nc_gmmk} is not able to exactly capture the regime when $d/n$ is larger than $1$ and below $(K+1)/n$. The further improvements of the bound will rely on a much more refined analysis to understand the feasibility of~\eqref{numer_fea1}
\subsection{${\cal NC}$ for two-layer neural network}
Theorem~\ref{thm:nc_gmm2} and \ref{thm:nc_gmmk} only provide an answer to whether a neural collapse configuration exists for a two-layer neural network. Note that the objective function~\eqref{def:dnn2} is non-convex for the training of two-layer neural networks with $\ReLU$ activation. This subsection is devoted to exploring whether such a network could converge to a neural collapse configuration by SGD. Similar to the previous subsection, we sample data with $K=2$ from GMM under different $d$ and $\sigma$ and train two-layer $\ReLU$ neural networks to do the classification task. We train each network for $10^5$ epochs under the cross-entropy loss~\eqref{def:dnn2} with activation regularization: we use plain SGD starting with learning rate $0.1$ and divide the rate by $10$ at the $1/3$ and $2/3$ of the first epoch respectively and the regularization parameters are $\lambda_{W}=10^{-3}$ and $ \lambda_{H} = 10^{-6}$.
We fix $\bmu_1=-\bmu_2=\be_1$ and let $d/n \in \left[1.1,1.3,1.5,2,3,4 \right]$ and $\sigma \in \left[0.18,0.36,0.53,0.8,1.07,1.42\right]$ which correspond to the grid in Figure~\ref{fig:collapse_tc}. We calculate the commonly adopted $\mathcal{NC}_1$ metric to measure the degree of feature collapse for each network after training.
\begin{equation}\label{eq:sampleconvergencemetric}
 \mathcal{NC}_1 := \frac{1}{K} \Tr\left({\BSigma_W\BSigma_B^{\dagger}}\right),
\end{equation}
where 
\[
\BSigma_W := \frac{1}{N} \sum_{k=1}^K\sum_{i=1}^{n} (\bh_{ki}-\bar{\bh}_k) (\bh_{ki}-\bar{\bh}_k)^{\top},~~ \BSigma_{B} := \frac{1}{K}\sum_{k=1}^K (\bar{\bh}_k-\bh_G)(\bar{\bh}_k-\bh_G)^{\top}
\]
and
\[
\bh_G := \frac{1}{K}\sum_{k=1}^{K} \bar{\bh}_{k},\quad \bar{\bh}_k := \frac{1}{n}\sum_{i=1}^{n} \bh_{ki},~~1\leq k\leq K.
\]
\begin{figure}[h]
  \centering
   \includegraphics[width=130mm]{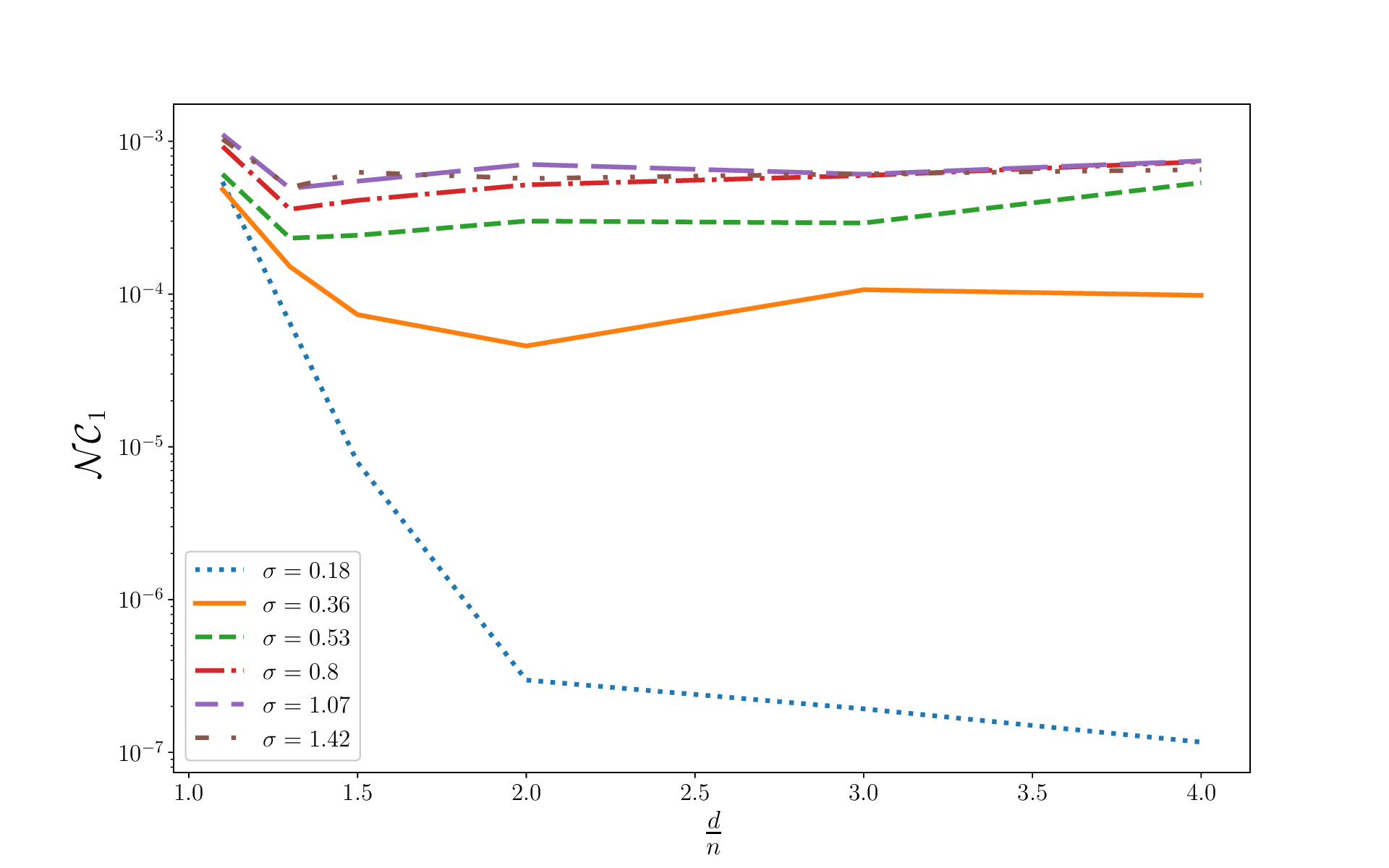}
  \caption{${\cal NC}_1$ plot for two-layer ReLU networks after training for $10^6$ epochs with varying $d$ and variance $\sigma$. The legend denotes the variance $\sigma$ in GMM models that generate data. }
  \label{fig:twol_collapse_mc}
\end{figure}

We observe that when data is well clustered ($\sigma =0.18$), the $\mathcal{NC}_1$ drops as the data dimension $d$ increases. In particular, the $\mathcal{NC}_1$'s are all below $10^{-4}$ except for the case when $d/n=1.1$, which is consistent with our characterization in Figure~\ref{fig:collapse_tc} and under this noise level we find that large $d$ facilitates the feature collapse. However, as $\sigma$ increases, larger $d$ does not yield to smaller ${\cal NC}_1$ anymore. Especially, for cases where $\sigma \geq 0.8$ (SNR is close or smaller than $1$), we find the ${\cal NC}_1$ almost all stay at the same level of magnitude. We think a higher $\sigma$ makes the gradient less aligned among the SGD iterates. Also, increasing the dimension of the data means more neurons need to be aligned to induce feature collapse, which could result in slow convergence to neural collapse configurations or getting stuck at local minima. We verify this analysis by plotting the following quantities together with $\mathcal{NC}_1$ along training epochs, which measure the convergence of mean and weight to neural collapse configurations~\eqref{eq:barH_ce} in Figures~\ref{fig:twol_collapse_mc_ana1} and~\ref{fig:twol_collapse_mc_ana2},
\begin{equation}\label{def:rela_error}
  \begin{aligned}
\mathcal{NC}_{2,\bar{\BH }} &= \left\|\frac{\bar{\BH}^\top\bar{\BH}}{\left\|\bar{\BH}^\top\bar{\BH}\right\|_F}-\frac{1}{\sqrt{K}} \BI_K \right\|_F, \\
\mathcal{NC}_{2,\BW} &= \left\|\frac{\BW^\top\BW}{\left\|\BW^\top\BW\right\|_F}-\frac{1}{\sqrt{K-1}} \BC_K \right\|_F, \\
\mathcal{NC}_{3} &= \left\|\frac{\BW^\top\bar{\BH}}{\left\|\BW^\top\bar{\BH}\right\|_F}- \frac{1}{\sqrt{K-1}} \BC_K \right\|_F, \\
  \end{aligned}
\end{equation}
where $\BC_K = \I_K - \BJ_K/K$ is defined in~\eqref{def:ck}

\begin{figure}[h]
  \centering
   \includegraphics[width=160mm]{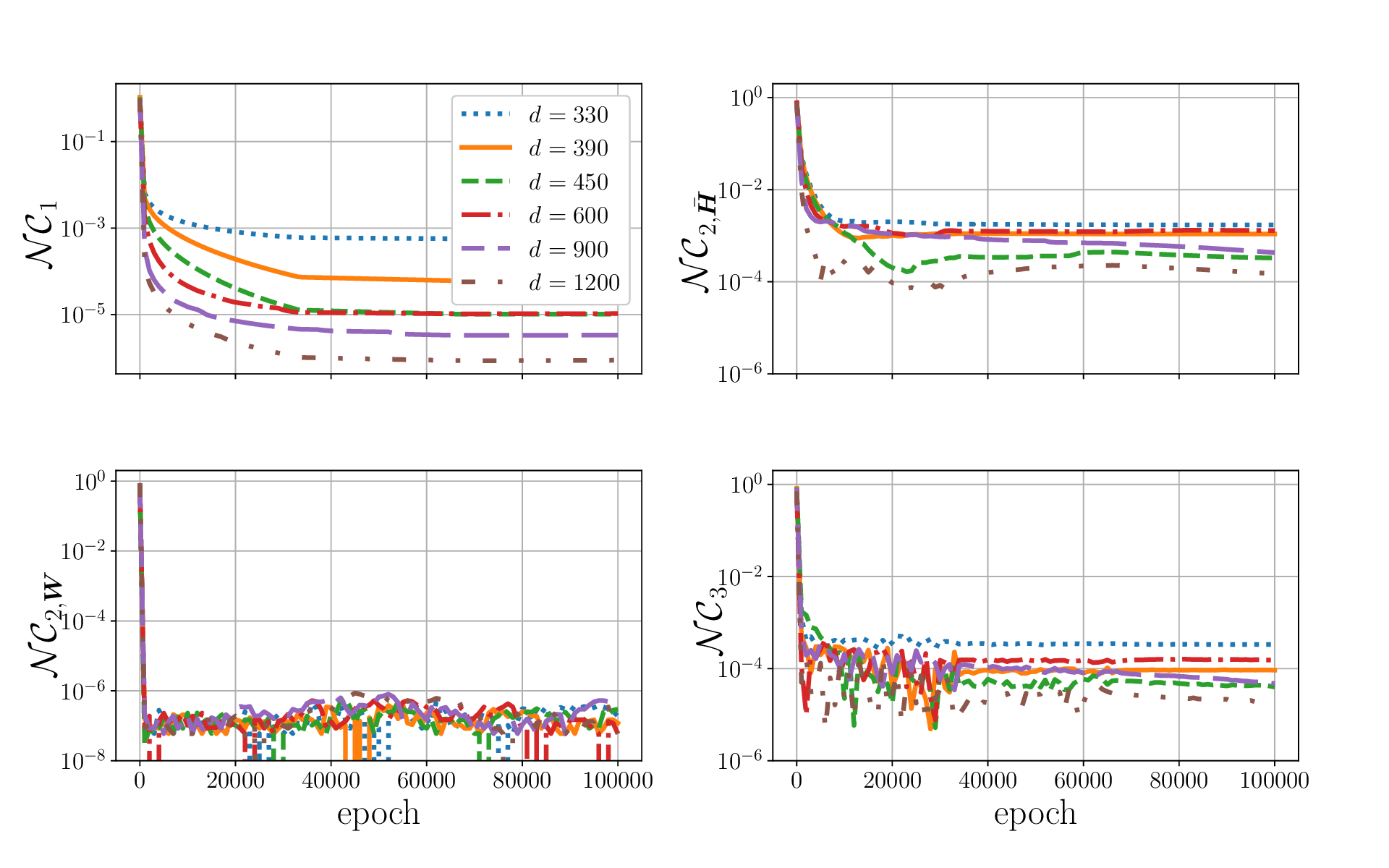}
  \caption{${\cal NC}_1$-${\cal NC}_3$ of two-layer networks along the training epochs in Figure~\ref{fig:twol_collapse_mc} along the training epochs with $\sigma=0.18$. The legend denotes the dimension $d$ of the training data. Upper left plot: $\mathcal{NC}_1$. The left three plots record the alignment of mean feature matrix $\bar{\BH}$ and weight of classifier $\BW$ to orthogonal frame and ETF (described by~\eqref{eq:barH_ce}) respectively measured by the relative error~\eqref{def:rela_error} under Frobenius norm.}
  \label{fig:twol_collapse_mc_ana1}
\end{figure}

\begin{figure}[h]
  \centering
   \includegraphics[width=160mm]{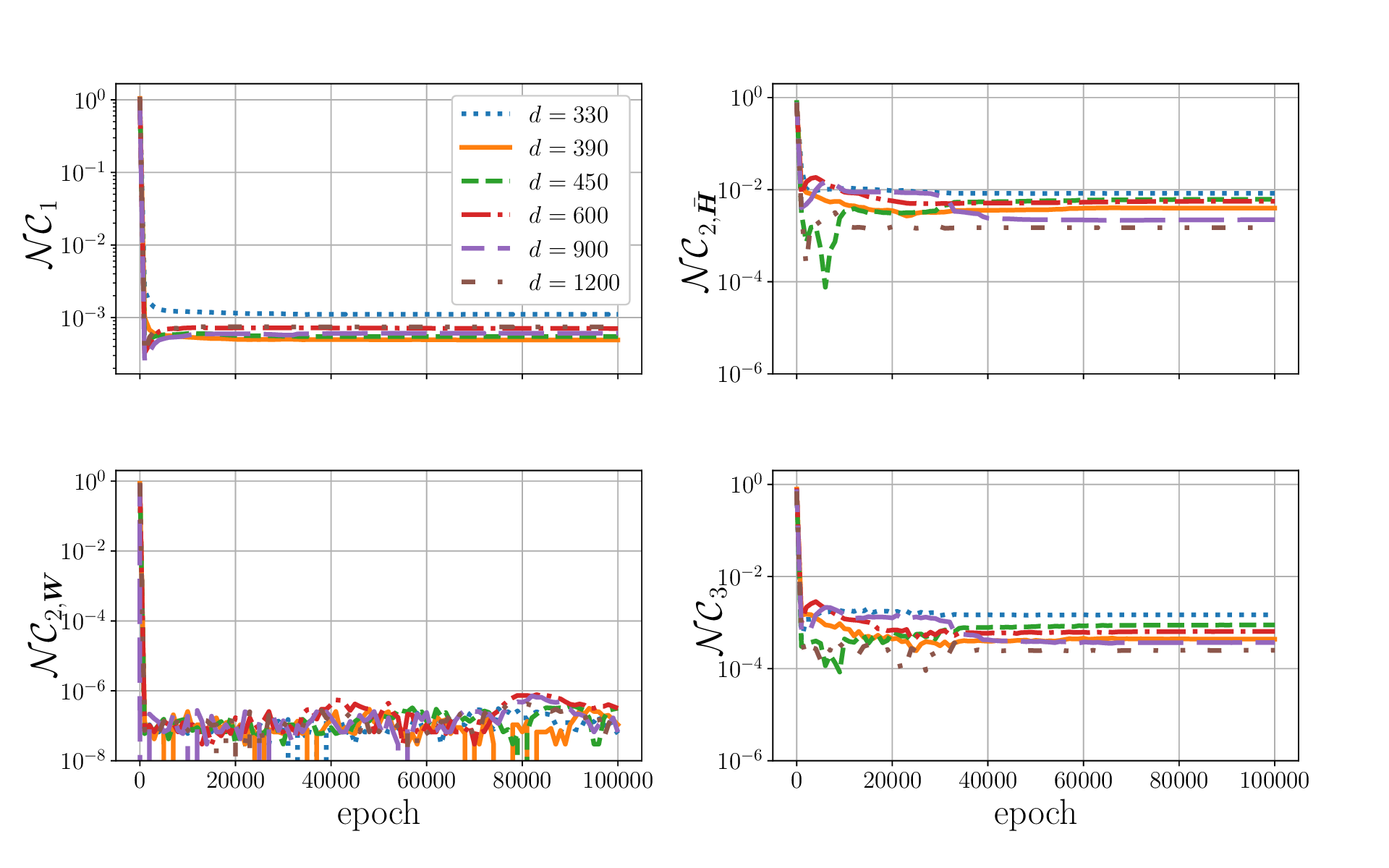}
  \caption{${\cal NC}_1$-${\cal NC}_3$ of two-layer networks along the training epochs in Figure~\ref{fig:twol_collapse_mc} with $\sigma=1.42$.}
  \label{fig:twol_collapse_mc_ana2}
\end{figure}

Figures~\ref{fig:twol_collapse_mc_ana1} and~\ref{fig:twol_collapse_mc_ana2} respectively depict the change of $\mathcal{NC}_{1-3}$ metrics of networks shown in Figure~\ref{fig:twol_collapse_mc} when $\sigma$ equals to $0.18$ and $1.42$. We observe that all the metrics converge for all the networks after $10^5$ training epochs, which implies the stabilization of the training process. The major difference between the final trained parameters and the neural collapse configuration comes from $\bar{\BH}$ ($\mathcal{NC}_{2,\bar{\BH}}$) and the concentration of feature vectors around the class mean feature vectors ($\mathcal{NC}_1$) while $\mathcal{NC}_{2,\BW}$ drops to the level below $10^{-6}$ for all the networks ($\mathcal{NC}_{2,\BW}$ actually becomes $0$ under single-precision floating-point format at some points). Although adding dimension in general makes $\mathcal{NC}_{2,\bar{\BH}}$ smaller, it also makes the decay of $\mathcal{NC}_{1}$ much harder in the case of $\sigma=1.42$, as $\mathcal{NC}_1$ quickly becomes stable for all networks in Figure~\ref{fig:twol_collapse_mc_ana2}. Additionally, by comparing Figures~\ref{fig:twol_collapse_mc_ana1} and~\ref{fig:twol_collapse_mc_ana2}, $\mathcal{NC}_1$ and $\mathcal{NC}_{2,\bar{\BH}}$ both become higher for $\sigma = 1.42$, which implies that the convergence becomes slower as the noise level rises. This observation leads us to conclude that when the GMM model has a low noise level, SGD could approach a neural collapse configuration as the global minimizer when it exists. However, when the noise level is high, it takes a very long time for the SGD to get close to a neural collapse configuration.

\subsection{${\cal NC}$ for three-layer neural network}
As shown in Theorem~\ref{Thm:three_layer_main}, for a three-layer neural network with the first layer randomly initialized, ${\cal NC}$ should emerge if the width of first layer $d_1\gtrsim N\log N$: 
$\BW_1 \in \mathbb{R}^{d_1 \times d}$ to be Gaussian matrix containing i.i.d entries from $\mathcal{N}(0,1/\sqrt{d_1})$. 
To better understand our theoretical result, we train three-layer neural networks with the weights on the first-layer ($\BW_1$) fixed. The dataset is FMnist containing $n=500$ samples from each class with $K=10$. 

 We train each network under the regularized cross-entropy loss~\eqref{def:dnn3} with the same setting in terms of training epoch, regularization parameters, and stepsize schedule as those in the previous subsection. In Figure~\ref{fig:tlnn}, we plot again the metrics $\mathcal{NC}_{1-3}$ along the training epochs.
\begin{figure}[h!]
  \centering
   \includegraphics[width=160mm]{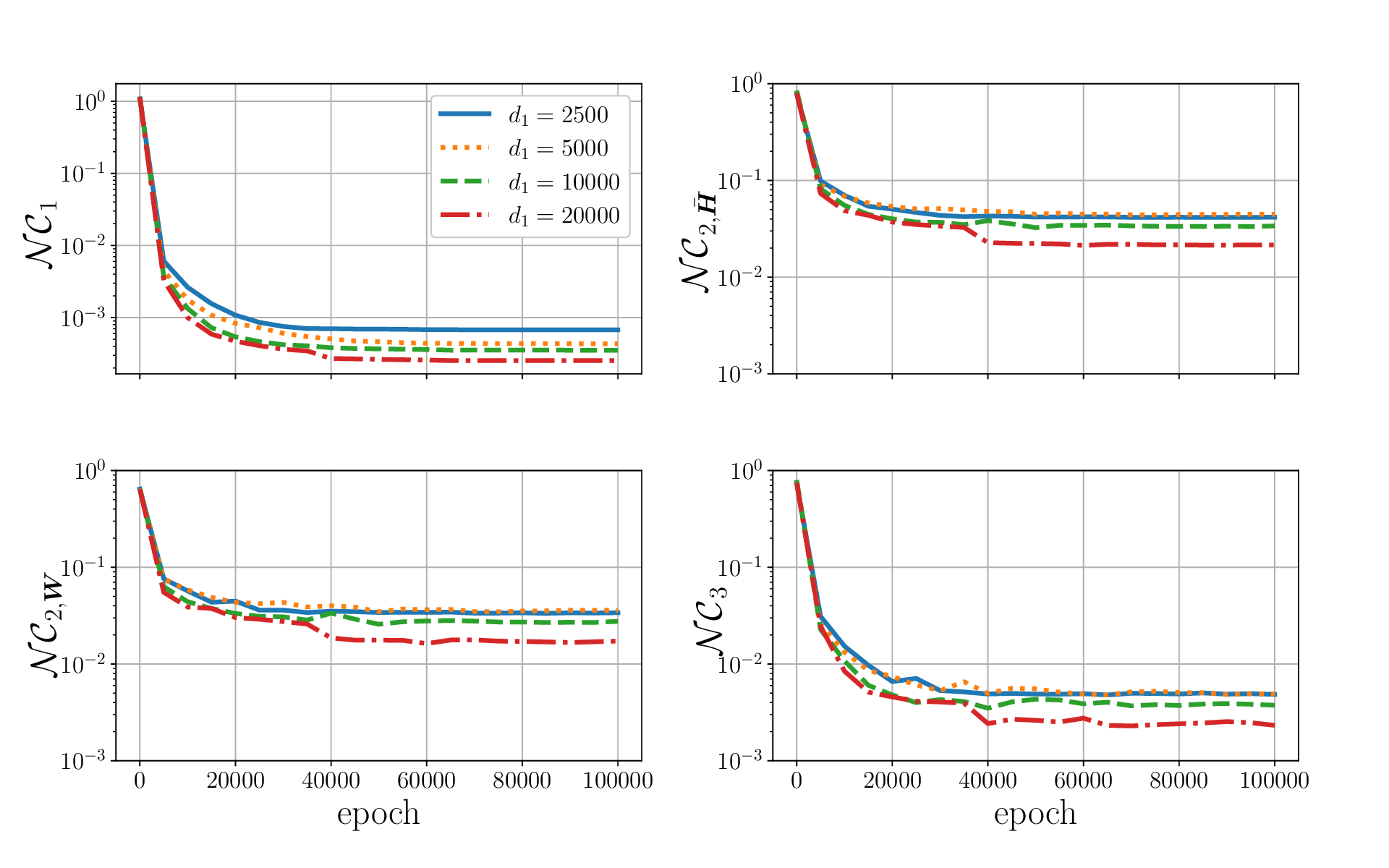}
  \caption{${\cal NC}_1$-${\cal NC}_3$ of three-layer networks along the training epoch. The legend denotes the width $d_1$ of the first-layer weight. }
  \label{fig:tlnn}
\end{figure} 

As $d_1$ increases, the feature collapse and the convergence of $\bar{\BH}$ and $\BW$ to the ${\cal NC}_2$ characterized in~\eqref{eq:barH_ce} are stronger in the terminal phase of training. The relative error proposed in~\eqref{def:rela_error} shrinks as the width $d_1$ increases but it is still at the order of $10^{-3}$ for $d_1=2\cdot 10^{4}$. This implies despite the global optimality of ${\cal NC}$ in the regularized ERM, it may still take much time for the actual training process to achieve ${\cal NC}.$
Similar to cases presented in the previous subsection, the slow convergence is likely due to the nonconvex nature of the objective function and the random feature training. The gap between the experiments and theory  exhibited in this section calls for further studies of the convergence of first-order iterative algorithms to the neural collapse configuration. Our code for all the experiments above is available on \href{https://github.com/WanliHongC/Relu-NC}{Github}.

\section{Proofs}\label{s:proof}

\subsection{Unconstrained positive feature models}\label{ss:upfm}

\begin{proof}[\bf Proof of Theorem~\ref{thm:upfm}(a)]

Consider
\[
R_N(\BW,\BH) = \frac{1}{N} \ell_{CE}(\BW^{\top}\BH,\BY) + \frac{\lambda_W}{2}\|\BW\|_F^2 + \frac{\lambda_H}{2}\|\BH\|_F^2
\]
subject to $\BH\geq 0$ and $\BY$ is the label matrix defined in~\eqref{def:Y}. This optimization problem is nonconvex. The proof idea is: we will find a convex relaxation of $R_N(\BW,\BH)$, then find the corresponding global minimizer, and then prove that the solutions to the original nonconvex problem and convex relaxation are identical. Note that $R_N(\BW,\BH)$ only depends on $\BZ: = \BW^{\top}\BH$, $\BU:=\BW^{\top}\BW$, and $\BV:= \BH^{\top}\BH$. Therefore, a convex relaxation is given by
\[
\min~ \frac{1}{N} \ell_{CE}(\BZ,\BY) + \frac{\lambda_W}{2}\Tr(\BU) + \frac{\lambda_H}{2}\Tr(\BV)
\]
subject to
\[
\BQ: = \begin{bmatrix}
\BU & \BZ \\
\BZ^{\top} & \BV
\end{bmatrix} \succeq 0,~~~\BV\geq 0
\]
where $\BU\in\RR^{K\times K}$ and $\BV\in\RR^{N\times N}.$ 
We claim that
\[
\BU = \frac{a^2}{b} \BC_K,~~\BV = b\BY^{\top}\BY,~~\BZ = a\BC_K\BY
\]
for some $a > 0$ and $b> 0$ where $\BC_K$ is the centering matrix in~\eqref{def:ck}.
Then it holds that
\begin{equation}\label{eq:kktQ}
\BQ = \begin{bmatrix}
\BU & \BZ \\
\BZ^{\top} & \BV
\end{bmatrix} = \begin{bmatrix}
\frac{a^2}{b}\BC_K & a \BC_K \BY \\
a\BY^{\top}\BC_K & b \BY^{\top}\BY
\end{bmatrix} = \frac{1}{b}
\begin{bmatrix}
a\BC_K \\
b\BY^{\top}
\end{bmatrix}
\begin{bmatrix}
a\BC_K \\
b\BY^{\top}
\end{bmatrix}^{\top}
\end{equation}
is exactly rank-$K$.

The Lagrangian is
\[
L(\BU,\BV,\BZ,\BS,\BB) = \frac{1}{N} \ell_{CE}(\BZ,\BY) + \frac{\lambda_W}{2}\Tr(\BU) + \frac{\lambda_H}{2}\Tr(\BV) - \left\lag \BS, \begin{bmatrix}
\BU & \BZ \\
\BZ^{\top} & \BV
\end{bmatrix}\right\rag - \frac{1}{2}\left\lag \BV,\BB\right\rag
\]
where $\BB\geq 0$ and $\BS\in\RR^{(K+N)\times (K+N)}$ is positive semidefinite:
\[
\BS 
= \begin{bmatrix}
\BS_{11} & \BS_{12} \\
\BS_{21} & \BS_{22}
\end{bmatrix} \succeq 0.
\]
The dimension of the blocks of $\BS$ matches that of $\BU,\BV,$ and $\BZ.$

Setting $\nabla_{\BU,\BV,\BZ} L(\BU,\BV,\BZ,\BS,\BB)=0$ leads to
\[
\lambda_W \I_K = 2 \BS_{11},~~\lambda_H \I_{N} = 2\BS_{22} + \BB,~~2\BS_{12} = \frac{1}{N} (\BP - \BY)
\]
where
\[
\frac{\pa \ell_{CE}(\BZ,\BY)}{\pa \bz_{ki}} = \frac{1}{N} (\bp_{ki} - \be_k),~~~
\bp_{ki} = \frac{\exp(\bz_{ki})}{\lag \exp(\bz_{ki}), \bone_K\rag}.
\]
We have the dual variable
\[
\BS = \frac{1}{2}
\begin{bmatrix}
\lambda_W\I_K & N^{-1}(\BP-\BY) \\
N^{-1}(\BP-\BY)^{\top} & \lambda_H\I_N - \BB 
\end{bmatrix}.
\]
To ensure $\BQ$ is a global minimizer, then there exist $\BS\succeq 0$ and $\BB\geq 0$ such that the following KKT condition holds
\begin{equation}\label{eq:kkt}
\BQ:= \begin{bmatrix}
\BU & \BZ \\
\BZ^{\top} & \BV
\end{bmatrix}\succeq 0,~~\BS\BQ = 0,~~\lag \BB, \BV\rag = 0.
\end{equation}
Under $\BZ = a \BC_K\BY$, we have
\[
\BP = \bar{\BP}\BY = \bar{\BP}\otimes\bone_n^{\top},\qquad \bar{\BP} = \frac{\BJ_K + (e^a - 1)\I_K}{K - 1 + e^a}
\]
and
\[
\I_K - \bar{\BP} = \frac{K\I_K - \BJ_K}{K-1+e^a} = \frac{K}{K-1+e^a}(\I_K - \BJ_K/K).
\]
As a result, we have
\[
\BS = \frac{1}{2}
\begin{bmatrix}
\lambda_W\I_K & -\frac{K}{N(K-1+e^a)}\BC_K\BY \\
-\frac{K}{N(K-1+e^a)} \BY^{\top}\BC_K & \lambda_H\I_N -\BB
\end{bmatrix}.
\]
Moreover, we choose
\begin{equation}\label{eq:kktB}
\BB = t (\BJ_N - \I_K\otimes \BJ_n) = t (\BJ_K -\I_K)\otimes \BJ_n.
\end{equation}
It remains to verify that~\eqref{eq:kkt} holds for some positive $a$, $b$, and $t$. Note that $\BQ\succeq 0$ and $\BB\geq 0$ follow from its construction~\eqref{eq:kktQ}. For $\lag \BB,\BV\rag$, we have
\[
\lag \BB, \BV\rag = bt \lag \BJ_N - \I_K\otimes\BJ_n, \BY^{\top}\BY\rag = bt (\lag \BJ_N, \I_K\otimes\BJ_n\rag - \|\I_K\otimes\BJ_n\|_F^2) =0,~~\forall b \geq 0,t\geq 0
\]
where
$\BY^{\top}\BY = (\I_K\otimes\bone_n)(\I_K\otimes\bone_n^{\top}) = \I_K\otimes\BJ_n.$ It suffices to ensure $\BS\succeq 0$ and $\BS\BQ = 0$.
Due to the factorization of $\BQ$ in~\eqref{eq:kktQ}, we have
\begin{align*}
& \BS \BQ = 0 \Longleftrightarrow
\begin{bmatrix}
\lambda_W\I_K & -\frac{K}{N(K-1+e^a)}\BC_K\BY \\
-\frac{K}{N(K-1+e^a)} \BY^{\top}\BC_K & \lambda_H\I_N -\BB
\end{bmatrix} 
\begin{bmatrix}
a \BC_K \\
b\BY^{\top}
\end{bmatrix} = 0,
\end{align*}
which is equivalent to
\begin{align*}
0& = a \lambda_W \BC_K - \frac{b\BC_K}{K-1+e^a},  \\
0& = -\frac{Ka}{N(K-1+e^a)} \BY^{\top} \BC_K + b\lambda_H\BY^{\top} - b \BB\BY^{\top} \\
& = -\frac{Ka}{N(K-1+e^a)} (\BY^{\top} - K^{-1} \BJ_{N\times K}) + b\lambda_H\BY^{\top} - bnt(\BJ_{N\times K} - \BY^{\top})\end{align*}
where $\BY = \I_K\otimes\bone_n^{\top},$ $\BY\BY^{\top} = n\I_K = (N/K)\I_K$ and
\[
t^{-1}\BB \BY^{\top} = (\BJ_N - \I_K\otimes \BJ_n)(\I_K\otimes\bone_n) = n\BJ_{N\times K} - n\I_K\otimes \bone_n = n(\BJ_{N\times K} - \BY^{\top}).
\]
To make $\BS\BQ=0$ hold, we need to have
\begin{align*}
 \frac{b}{K-1+e^a} & = a\lambda_W, \\
\frac{Ka}{N(K-1+e^a)} & = b\lambda_H + bnt, \\
\frac{a}{N(K-1+e^a)} & = bnt
\end{align*}
where $n = N/K.$
Then the second and third equations determine $t$:
\begin{equation}\label{eq:kktt}
K = \frac{\lambda_H}{nt} + 1\Longleftrightarrow \lambda_H + nt - nt K = 0 \Longleftrightarrow t = \frac{\lambda_H}{n(K-1)}.
\end{equation}
The coefficients $a$ and $b$ are the solution to 
\begin{equation}\label{eq:kktab}
 \frac{b}{K-1+e^a} = a\lambda_W, ~~\frac{a}{N(K-1+e^a)} = \frac{b\lambda_H}{K-1}~~ \Longrightarrow \frac{a}{b} =\sqrt{\frac{nK}{K-1}\cdot \frac{\lambda_H}{\lambda_W}}.
\end{equation}
Then
\[
a = \log \left( (K-1) \left[ \sqrt{ \frac{1}{nK(K-1)}\cdot\frac{1}{\lambda_H\lambda_W} } - 1\right]\right),~~~
b = \sqrt{\frac{K-1}{nK}\cdot\frac{\lambda_W}{\lambda_H}} a.
\]
Finally, we verify $\BS\succeq 0$:
\begin{align*}
\BS & = \frac{1}{2}\begin{bmatrix}
\lambda_W\I_K & -\frac{K}{N(K-1+e^a)}\BC_K\BY \\
-\frac{K}{N(K-1+e^a)} \BY^{\top}\BC_K & \lambda_H\I_N -\frac{\lambda_H}{n(K-1)} (\BJ_K-\I_K)\otimes\BJ_n
\end{bmatrix} 
\end{align*}
where
\[
\frac{K}{N(K-1+e^a)} = \frac{bnKt}{a} = Kn \cdot \sqrt{\frac{K-1}{nK}\cdot\frac{\lambda_W}{\lambda_H}} \cdot \frac{\lambda_H}{n(K-1)} = \sqrt{\frac{K\lambda_W\lambda_H }{n(K-1)}}.
\]
Note that $\BS_{11} = \frac{\lambda_W}{2}\I_K\succ 0$, and thus $\BS\succeq 0 $ equals $\BS_{22} - \BS_{21}\BS_{11}^{-1}\BS_{12} \succeq 0$, i.e.,
\begin{align*}
2(\BS_{22} - \BS_{21}\BS_{11}^{-1}\BS_{12}) 
& = \lambda_H\I_N -\frac{\lambda_H}{n(K-1)} (\BJ_K-\I_K)\otimes\BJ_n - \frac{K^2}{N^2(K-1+e^a)^2\lambda_W} \BY^{\top}\BC_K\BY \\
& = \lambda_H\I_N -\frac{\lambda_H}{n(K-1)} (\BJ_K-\I_K)\otimes\BJ_n - \frac{K\lambda_H }{n(K-1)} \left( \I_K \otimes\BJ_n - \frac{\BJ_N}{K} \right) \\
& = \lambda_H\I_N - \frac{\lambda_H }{n} \left( \I_K \otimes\BJ_n \right)\succeq 0
\end{align*}
where 
\[
\BY^{\top}\BC_K\BY = \BY^{\top}\BY - \frac{1}{K}\BY^{\top}\BJ_K\BY = \I_K \otimes\BJ_n - \frac{\BJ_N}{K}.
\]

To summarize, we have verified that~\eqref{eq:kktQ} with coefficients~\eqref{eq:kktt} and~\eqref{eq:kktab} is the unique global minimizer where the uniqueness comes from the strict complementary slackness. Therefore, the global minimizer $(\BW,\BH)$ is given by 
\[
\BW^{\top}\BW = \frac{a^2}{b}\BC_K, \BH^{\top}\BH = b\BY^{\top}\BY \Longrightarrow \BW = \frac{a}{\sqrt{b}}\BC_K,~~\BH = \sqrt{b} \BY
\]
where $a$ is strictly positive if 
\[
(K-1) \left[ \sqrt{ \frac{1}{nK(K-1)}\cdot\frac{1}{\lambda_H\lambda_W} } - 1\right] > 1 \Longleftrightarrow 
 \sqrt{ \frac{K-1}{nK}\cdot\frac{1}{\lambda_H\lambda_W} } > K \Longleftrightarrow \sqrt{ \frac{n(K-1)}{K} } > N \lambda_Z.
\]
\end{proof}

\begin{proof}[\bf Proof of Theorem~\ref{thm:upfm}(b)]
Under $\ell_2$-loss function and unconstrained positive feature model, the empirical risk function is
\[
R_N(\BH_+,\BW) = \frac{1}{2N}\|\BW^{\top}\BH_+ -\BY\|_F^2 + \frac{\lambda_W}{2}\|\BW\|_F^2 + \frac{\lambda_H}{2} \|\BH_+\|_F^2
\]
where $\BH_+$ is a nonnegative matrix. 
We note that
\[
\min_{\BW,\BH} R_N(\BW,\BH) \leq \min_{\BW,\BH_+} R_N(\BW,\BH_+).
\]
Without any constraint on $\BH_+$, the minimization of $R_N(\BH,\BW)$ is exactly singular value thresholding. Let $\BZ = \BW^{\top}\BH$, and then
\[
\min_{\BW^{\top}\BH=\BZ} \frac{\lambda_W}{2}\|\BW\|_F^2 + \frac{\lambda_H}{2} \|\BH\|_F^2 = \sqrt{\lambda_W\lambda_H}\|\BZ\|_*.
\]
Thus $R_N(\BW,\BH)$ equals
\[
\min_{\BZ}\frac{1}{2N}\|\BZ - \BY\|_F^2 + \lambda\|\BZ\|_*
\]
and the global minimum for $R_N(\BW,\BH)$ is attained at
\[
\BZ = \BW^{\top}\BH = \frac{(\sqrt{n} - N\lambda)_+}{\sqrt{n}} \BY = (1 - \sqrt{n}K\lambda)_+\BY
\]
where $\BY$ is a binary matrix satisfying $\BY\BY^{\top} = n\I_K$ and $\lambda = \sqrt{\lambda_W\lambda_H}$. Given $\BZ$,~\cite[Lemma A.3]{ZDZ21} implies that the global minimizer is given by $(\BH,\BW)$ that satisfies
\begin{equation*}
 \begin{aligned}
 & \BH^{\top}\BH = \sqrt{\frac{\lambda_W}{\lambda_H}}\frac{(1- \sqrt{n}K\lambda)_+}{\sqrt{n}} \BY^{\top}\BY = \sqrt{\frac{\lambda_W}{\lambda_H}}\frac{(1- \sqrt{n}K\lambda)_+}{\sqrt{n}}\I_K\otimes\BJ_n, \\
 & \BW^{\top}\BW = \sqrt{\frac{\lambda_H}{\lambda_W}}n^{1/2} (1 - \sqrt{n}K\lambda)_+ \I_K,
 \end{aligned}
\end{equation*}
and we can select a nonnegative $\BH$ to be the global minimizer for~\eqref{def:upfm_l2}, for example,
\[
\BH= \bar{\BH} \otimes \bone_n^\top,~~\bar{\BH}=\left(\frac{\lambda_W}{n\lambda_H}\right)^{\frac{1}{4}}\begin{bmatrix} (1 - \sqrt{n}K\lambda)_+^{\frac{1}{2}} \BI_{K} \\ \bzero_{D-K,K} \end{bmatrix}.
\]

\end{proof}


\subsection{Two-layer neural network}\label{ss:2nn_fp}

\subsubsection{Proof of Theorem~\ref{thm:nc_gmm2}: ${\cal NC}$ for GMM with two clusters for $2n\geq d > n$}\label{sss:nc_gmm2}

\begin{proof}[\bf Proof of Theorem~\ref{thm:nc_gmm2}]
For now, we first try to establish a condition such that 
\begin{equation}\label{eq:feas_two_cluster}
\lag \bmu_1, \bbeta\rag + \sigma \BZ_1\bbeta = 1,~~~\lag \bmu_1,\bbeta\rag = 1,~~~\lag \bmu_2, \bbeta\rag + \sigma \BZ_2\bbeta \leq 0,
\end{equation}
is feasible, i.e., the set is non-empty. Once it is done, the same result also applies to the second cluster. The second equality constraint makes the argument cleaner without compromising much.
With~\eqref{eq:feas_two_cluster}, we have
\[
\BZ_1\bbeta = 0,~~~\lag \bmu_1,\bbeta\rag = 1,~~~\lag \bmu_2, \bbeta\rag + \sigma \BZ_2\bbeta \leq 0,
\]
which means $\bbeta$ is in the null space of $\BZ_1\in\RR^{n\times d}.$

Let $\BPhi \in \mathbb{R}^{d\times (d-n)}$ be a partial orthonormal matrix with columns being the basis of the null space of $\BZ_1$, which belongs to the Grassmannian $G_{d,d-n}.$
It is easy to see that $\bbeta$ has the following form:
\[
\bbeta = \BPhi\left(\frac{\BPhi^\top\bmu_1}{\|\BPhi^\top\bmu_1 \|^2} + \bv\right)
\]
where $\bv$ is perpendicular to $\BPhi^{\top}\bmu_1$, i.e., $\lag\bv,\BPhi^\top \bmu_1 \rag = 0.$

With the representation of $\bbeta$, it remains to ensure the third constraint in~\eqref{eq:feas_two_cluster} satisfies. Note that for any fixed $\bbeta$, 
\[
\lag \bmu_2,\bbeta\rag\bone_n + \sigma\BZ_2\bbeta \sim \mathcal{N}(\lag \bmu_2,\bbeta\rag\bone_n, \sigma^2 \|\bbeta\|^2).
\]
To guarantee a high probability of the event $\lag \bmu_2,\bbeta\rag\bone_n + \sigma\BZ_2\bbeta \leq 0$, we need to maximize the ratio of $\lag \bmu_2,\bbeta\rag$ over $\|\bbeta\|$ subject to the constraints on $\bbeta.$ Therefore, we define
\begin{align*}
f(\bv):=\frac{\lag \bmu_2,\bbeta(\bv)\rag}{\|\bbeta(\bv)\|} 
= \frac{\lag \BPhi^\top\bmu_1/\|\BPhi^\top\bmu_1 \|^2 + \bv, \BPhi^{\top}\bmu_2\rag}{ \sqrt{1/\|\BPhi^\top\bmu_1 \|^2 + \|\bv\|^2} }
\end{align*}
subject to $\lag \bv, \BPhi^{\top}\bmu_1\rag = 0.$ 
By applying Lemma~\ref{lemma:featc_opt} with $\bv_1:= \BPhi^\top \bmu_1/\|\BPhi^{\top}\bmu_1\|^2$ and $ \bv_2:= \BPhi^\top \bmu_2$, we have
\begin{equation}
 f^* := \min_{\lag \bv,\BPhi^\top\bmu_1\rag=0} f(\bv)
 =
  \begin{cases} 
 - \|\BPhi^{\top}\bmu_2\| \sqrt{1 - \cos^2\widetilde{\theta}}, ~~ &\text{if} ~~ \cos\widetilde{\theta} \geq 0, \\ 
 -\|\BPhi^\top\bmu_2\|, ~~ &\text{if} ~~ \cos\widetilde{\theta} < 0,
  \end{cases}
\end{equation} 
where $\widetilde{\theta}$ is the angle between $\BPhi^{\top}\widetilde{\bmu}_1$ and $\BPhi^{\top}\widetilde{\bmu}_2.$


Therefore, it suffices to have i.i.d. $n$ samples from $\mathcal{N}(f^*,\sigma^2)$ such that all of them are nonpositive.
By taking union bound over $n$ samples, we have
\begin{equation}\label{featc:noise_bound}
2\sigma \sqrt{\log{n}} + f^* \leq 0 
\end{equation}
with probability at least $1-O(n^{-1})$. It remains to estimate $f^*$.
From the lemma~\ref{lemma:JLangle}, 
\[
|\cos \widetilde{\theta} - \cos\theta|\leq \frac{4\eps}{(1-\eps)^2}.
\]
For $\cos\theta< -4\eps/(1-\eps)^2$, then $\cos\widetilde{\theta} < 0$, and then
\[
-f^* \geq (1-\eps)\sqrt{\frac{d-n}{d}}\|\bmu_2\|
\]
where the bound on $\|\BPhi^{\top}\bmu_2\| \geq (1-\eps) \sqrt{(d-n)/d}\|\bmu_2\|$ follows from the Johnson-Lindenstrauss lemma.
For $\cos\theta> -4\eps/(1-\eps)^2$ and $|\cos\widetilde{\theta}| \leq |\cos\theta| + 4\eps/(1-\eps)^2$, then
\begin{align*}
-f^* & = \sqrt{\|\BPhi^{\top}\bmu_2\|^2 - |\lag \BPhi^{\top}\bmu_1, \BPhi^{\top}\bmu_2\rag|^2/\|\BPhi^{\top}\bmu_1\|^2} \\
& = \|\BPhi^{\top}\bmu_2\| \sqrt{1- \cos^2\widetilde{\theta}} \geq (1-\eps)\sqrt{\frac{d-n-1}{d}}\|\bmu_2\| \sqrt{ 1- \left( |\cos\theta| + \frac{4\eps}{(1-\eps)^2}\right)^2 }.
\end{align*}
All bounds above hold with probability at least $1-C\exp(-\eps^2 (d-n-1)).$ 

For $d\geq 2n$, we notice that $\BX^{\top}$ is full row rank for any $\sigma>0$, and thus one can always find a $\bbeta$ such that $\BX_1^{\top}\bbeta = 1$ and $\BX_2^{\top}\bbeta \leq 0.$ 
\end{proof}

\subsubsection{Proof of Theorem~\ref{thm:nc_gmmk}(a): ${\cal NC}$ for GMM with multiple classes for $Kn\geq d > n$}

\begin{proof}[\bf Proof of Theorem~\ref{thm:nc_gmmk}(a)]

To ensure the neural collapse occurs, we first show that for each fixed $k$, there exists $\bbeta$ such that 
\[
(\bone_n\bmu_k^{\top} + \sigma\BZ_k) \bbeta = \bone_n,~~ \lag\bmu_k, \bbeta \rag= 1,~~ \lag\bmu_{k'}, \bbeta \rag= -\gamma~~(k'\neq k, \gamma>0).
\]
These equations reduce to the set such that
\[
\BZ_k \bbeta = 0,~~~\BPi \bbeta = \be_k - \gamma(\bone_K - \be_k) 
\]
where $\BPi\in\RR^{K\times d}$ and $\BZ_k \in\RR^{n\times d}.$ 
Let $\BPhi$ be an $d\times (d-n)$ orthogonal matrix whose columns span the null space of $\BZ_k$. Let
$\bbeta = \BPhi\balpha$
where $\balpha\in\RR^{d-n}$ and $\BPhi\in\RR^{d\times (d-n)}$ is a random projection matrix, and $\BZ_k \BPhi = 0.$
Since $\BPi$ is assumed to be of rank-$K$, 
\begin{equation}\label{mc:feabb}
\BPi\BPhi \balpha = \be_k - \gamma(\bone_K - \be_k)
\end{equation}
and $\BPi\BPhi\in\RR^{K\times (d-n)}$. Under $d-n>K$, then $\BPi\BPhi$ is of rank $K$ with probability 1. Here we choose
\begin{align*}
\balpha & = \BPhi^{\top}\BPi^{\top} (\BPi\BPhi \BPhi ^{\top}\BPi^{\top})^{-1}(\be_k - \gamma(\bone_K - \be_k)), \\
\bbeta & = \BPhi\BPhi^{\top}\BPi^{\top} (\BPi\BPhi \BPhi ^{\top}\BPi^{\top})^{-1}(\be_k - \gamma(\bone_K - \be_k)),
\end{align*}
where
\[
\|\bbeta\|^2 = (\be_k - \gamma(\bone_K - \be_k))^{\top} (\BPi\BPhi \BPhi ^{\top}\BPi^{\top})^{-1}(\be_k - \gamma(\bone_K - \be_k)).
\]
Note that for each given $\gamma$, 
\[
(\bone_n\bmu_{k'}^{\top} + \sigma\BZ_{k'})\bbeta \sim \mathcal{N}(-\gamma, \sigma^2\|\bbeta\|^2\I_n)
\]
for any $k'\neq k$, and $\lag \bmu_{k'}, \bbeta\rag = -\gamma.$
We proceed to determine the choice of $\gamma > 0$ by minimizing
\[
\min_{\gamma>0} -\frac{\gamma}{\|\bbeta\|} \Longleftrightarrow \max_{\gamma>0} F(\gamma)
\]
where
\begin{equation}
F(\gamma):= \frac{\gamma^2}{\|\bbeta\|^2}= \frac{\gamma^2}{(\be_k - \gamma(\bone_K - \be_k))^{\top} (\BPi\BPhi \BPhi ^{\top}\BPi^{\top})^{-1}(\be_k - \gamma(\bone_K - \be_k)) }.
\end{equation}
Then we know that $(\bone_n\bmu_{k'}^{\top} +\sigma\BZ_{k'})\bbeta \leq 0$ holds if
\[
2\sigma\sqrt{\log (Kn)} \leq F(\gamma)
\]
with probability at least $1-O((Kn)^{-1}).$

By Lemma~\ref{lemma:JLbasis}, it holds with probability at least $1- c\exp(-\eps^2 (d-n))$ that
\[
\lambda_{\min}(\BPi\BPhi \BPhi ^{\top}\BPi^{\top}) \geq \left((1-\eps)^2-4K\eps\right)\frac{d-n}{d}\sigma^2_{\min}(\BPi)
\]
where $\sigma_{\min}(\BPi)$ is the smallest singular value of $\BPi$. Therefore, we have
\[
\|\bbeta\|^2 \leq \frac{(1+(K-1)\gamma^2)d}{\left((1-\eps)^2-4K\eps\right)(d-n)\sigma^2_{\min}(\BPi)},
\]
which implies
\[
\sup_{\gamma>0}F(\gamma) \geq \sup_{\gamma>0} \frac{\gamma^2}{\|\bbeta\|^2} = \frac{\left((1-\eps)^2-4K\eps\right)(d-n)\sigma^2_{\min}(\BPi) }{(K-1) d}.
\]
Then $(\bone_n\bmu_{k'}^{\top} + \sigma\BZ_{k'})\bbeta \leq 0 $
is implied by
\[
\sigma <\frac{1}{2} \sqrt{\frac{\left((1-\eps)^2-4K\eps\right)(d-n)\sigma^2_{\min}(\BPi) }{(K-1)d \log (Kn)}}
\]
with probability at least $1 - c\exp(-\eps^2(d-n)) - (Kn)^{-1}.$ Note that we need to choose $\eps < 1/(10K)$, i.e., $d-n \geq CK^2 \log Kn$ such that the probability is at least $1-O(n^{-1}).$
\end{proof}

\subsubsection{Proof of Theorem~\ref{thm:nc_gmmk}(b): ${\cal NC}$ for GMM with multiple classes under $d/n>(K+1)/2$}

\begin{proof}[\bf Proof of Theorem~\ref{thm:nc_gmmk}(b)]
In this subsection, we present a tighter bound for the linear feasibility when $d/n>(K+1)/2$ holds. For class $k$, we consider the set of $\bbeta$'s such that
\begin{equation}
 (\bone_n \bmu_k^\top + \sigma\BZ_k)\bbeta = \bone_n, ~~\lag\bmu_k, \bbeta\rag = 1, ~~ \lag \bmu_{k'}, \bbeta\rag \bone_n + \sigma\BZ_{k'}\bbeta\leq 0.
\end{equation}
We consider $\bbeta$ in the following form:
\begin{equation}
\bbeta = \BPhi \left( \bv + \frac{\BPhi^{\top}\bmu_k}{\|\BPhi^{\top}\bmu_k\|^2} \right)
\end{equation}
where $\BPhi\in\RR^{d\times (d-n)}$ represents the orthogonal basis of the null space of $\BZ_k$ and 
\[
\lag \bv, \BPhi^{\top}\bmu_{k}\rag = 0,~~~1\leq k\leq K.
\]
Such $\bv\in\RR^{d-n}$ exists if $d-n\geq K+1$.
Now for any $k'\neq k$, it holds that
\[
\lag \bmu_{k'}, \bbeta\rag \bone_n + \sigma\BZ_{k'}\bbeta = \left(\lag\BPhi^{\top}\bmu_{k'}, \BPhi\bmu_{k} \rag \bone_n+ \sigma \frac{\BZ_k \BPhi\BPhi^{\top}\bmu_k}{\|\BPhi^{\top}\bmu_k\|^2}\right) + \sigma \BZ_{k'} \BPhi \bv.
\]
To have $\lag \bmu_{k'}, \bbeta\rag \bone_n + \sigma\BZ_{k'}\bbeta < 0$ it suffices to find $\bv$ such that $\BZ_{k'}\BPhi\bv < 0$ for $k'$, since we can arbitrarily rescale the term $\BZ_{k'}\BPhi\bv$ by rescaling the norm of $\bv$.
The effective dimension of all $\bv$ is $d-n-K.$ Therefore, this problem is equivalent to finding $\bar{\bv}\in\RR^{d-n-K}$ such that
\[
\BA\bar{\bv} < 0
\]
where $\BA \in \RR^{(K-1)n\times (d-n-K)}$ is a Gaussian random matrix.
The next proposition provides sufficient conditions for constructing a high probability bound.

We aim to search for $\bar{\bv}\in\RR^{d-n-K}$ such that $\BA \bar{\bv}<0$ holds entrywisely. By hyperplane separation theorem, if the convex set $\{\BA^{\top}\bs: \|\bs\|_1=1,~\bs\geq 0,~\bs\in\RR^{(K-1)n}\}$ does not contain $0$, then there exists $\bar{\bv}$ s.t. 
  \[
 \lag \bar{\bv}, \BA^{\top}\bs\rag < 0,~\forall \|\bs\|_1=1,~\bs\geq 0 \Longleftrightarrow \BA^{\top}\bar{\bv} < 0.
  \]
 Hence it suffices to show that $\min_{\bs\geq 0,\|\bs\|_1=1} \|\BA^{\top}\bs\| > 0$ holds with high probability. Note that
  \begin{align*}
    &\min_{\bs\geq 0,\|\bs\|_1=1}\max_{\|\bu\|=1} \lag \bu,\BA^{\top}\bs \rag = \min_{\bs\geq 0,\|\bs\|_1=1} \|\BA^{\top}\bs\| = \min_{\bs\geq 0} \frac{\|\BA^{\top}\bs\| }{\|\bs\|_1} = \min_{\bs\geq 0} \frac{\|\BA^{\top}\bs\| }{\|\bs\|_2} \frac{\|\bs\|_2}{\|\bs\|_1} \\
  & \geq \frac{1}{\sqrt{(K-1)n}} \min_{\bs \geq 0} \frac{\|\BA^{\top}\bs\| }{\|\bs\|_2} = \frac{1}{\sqrt{(K-1)n}} \min_{\bs \geq 0,\|\bs\|=1} \|\BA^{\top}\bs\| \\
  & = \frac{1}{\sqrt{(K-1)n}} \min_{\bs \geq 0,\|\bs\|=1}\max_{\|\bu\|=1} \lag\bu,\BA^\top\bs \rag.
  \end{align*}
 Applying Proposition~\ref{Thm:Gordontess}, we have,
  \begin{equation}
 \mathbb{E}\min_{\bs \geq 0,\|\bs\|=1}\max_{\|\bu\|=1} \lag\bu,\BA^\top\bs \rag \geq \sqrt{d-n-K} - \sqrt{(K-1)n/2}.
  \end{equation}
 Therefore, we have
  \[
 \E\min_{\bs\geq 0,\|\bs\|_1=1} \|\BA^{\top}\bs\| \geq \frac{1}{\sqrt{(K-1)n}} \E\min_{\bs \geq 0,\|\bs\|=1} \|\BA^{\top}\bs\| \geq \frac{\sqrt{d-n-K} - \sqrt{(K-1)n/2}}{\sqrt{(K-1)n}}.
  \]
By Lemma~\ref{lemma:phi_2_concen}, we have
  \[
 \Pr\left(\min_{\bs \geq 0,\|\bs\|=1} \|\BA^{\top}\bs\| \leq \sqrt{d-n-K} - \sqrt{(K-1)n/2} - t\right) \leq \exp(-t^2/2).
  \]
 As a result, by taking $t=\sqrt{2\log{n}}$, it holds with probability at least $1-1/n$ that
  \begin{equation*}
    \begin{aligned}
  &\min_{\bs\geq 0,\|\bs\|_1=1} \|\BA^{\top}\bs\| \geq \frac{\sqrt{d-n-K} - \sqrt{(K-1)n/2}- \sqrt{2\log n}}{\sqrt{(K-1)n}}.
    \end{aligned}
  \end{equation*}
 Therefore, to ensure the existence of an ${\cal NC}$ solution, it suffices to have the RHS above positive:
  \[
 d- n-K \geq \frac{(K-1)n}{2} + 2\log n + 2\sqrt{(K-1)n\log n}.
  \]
\end{proof}

\subsection{Proof of Theorem~\ref{Thm:three_layer_main}: neural collapse for three-layer neural network}\label{ss:3layer}

Next, we consider the feasibility of inducing the neural collapse of a three-layer network with the first layer having random Gaussian weight, i.e. the format of the output is given by
\begin{equation}\label{eq:three_layernn}
f(\BW_1,\BW_2,\BW_3) = \BW_3^\top\sigma_{\ReLU}\left(\BW_2^\top\sigma_{\ReLU}\left(\BW_1^\top\BX\right)\right)
\end{equation}
where $\BW_1\in \mathbb{R}^{d \times d_1}, \BW_2\in \mathbb{R}^{d_1 \times d_2}, \BW_3\in \mathbb{R}^{d_2 \times K}$ and $\BX \in \mathbb{R}^{d \times N}$ is some general data.
One natural question is whether neural collapse occurs for a three-layer neural network. 
From Theorem~\ref{thm:NC_general}(b), we know that if there exists $\BW_1$ such that $\sigma_{\ReLU}(\BW_1^{\top}\BX)\in \RR^{d_1\times N}$ is rank $N$ with $d_1\geq N$, then there exists $\BW_2$ and $\BW_3$ such that the neural collapse occurs. 

We will show that by setting $\BW_1$ to be a Gaussian random matrix with each entry being i.i.d Gaussian random variables, i.e., $\sigma_{\ReLU}(\BW_1^{\top}\BX)$ gives the ReLU random feature of $N$ data points in $\RR^d$, then the neural collapse occurs with high probability if $d_1$ is sufficiently large.

Given $N$ points $\{\bx_i\}_{i=1}^N$, we define that
\[
[\BH_{\infty}]_{ij} : = \E_{\bz\sim\mathcal{N}(0,\I_d)} \left( \sigma_{\ReLU} (\lag \bz, \bx_i\rag) - \sqrt{\frac{2}{\pi}}\right)\left(\sigma_{\ReLU}(\lag \bz, \bx_j\rag) - \sqrt{\frac{2}{\pi}}\right)
\]
as the kernel matrix. Provided that $\bx_i$ is not parallel to any $\bx_j$, then it can be shown that $\BH_{\infty}(\BX)$ is full rank with $\lambda_{\min}(\BH_{\infty}) > 0$ as shown in the lemma below.

\begin{lemma}\label{lemma:tl_exp} 
Let $\BX \in \mathbb{R}^{d \times N}$ be a data matrix with columns $\bx_i \nparallel \bx_j$ for any $(i,j)$ pair. The kernel matrix
  \[
 \BH_{\infty}: = \mathbb{E}_{\bz \sim \mathcal{N}(0,\BI_d)}\sigma(\BX^{\top}\bz) \sigma(\BX^{\top}\bz)^{\top} - \frac{2}{\pi}\BJ_N
  \]
 has a strictly positive smallest singular value $\lambda_{\min}(\BH_{\infty}) > 0.$
\end{lemma}

\begin{proof}[\bf Proof of Lemma~\ref{lemma:tl_exp}] The main idea for this proof is adopted from~\cite[Theorem 3.1]{DZPS18}. Define the feature map $\phi_{\bx}(\bw):=\sigma(\bw^\top \bx) - \sqrt{2/\pi}$,~~$\bw\sim\mathcal{N}(0,\I_d)$, which is a continuous function w.r.t. $\bw\in\RR^d$ and the gradient $\nabla_{\bw}\sigma(\bw^{\top}\bx_i)$ is continuous everywhere except 
 \[
 D_i := \left\{\bw: \lag \bx_i,\bw\rag = 0 \right\}.
  \]
 
Now to prove $\BH_{\infty}$ is strictly positive definite, it is equivalent to show $\phi_{\bx_i}(\bw)$ is linearly independent, i.e., for any $\{a_i\}_{i=1}^N$ such that
  \[
  \sum_{i=1}^{N} a_i \phi_{\bx_i}\left(\bw \right) = 0,~~\forall \bw\sim\mathcal{N}(0,\I_d),
  \]
 then $a_i = 0$ holds for $1\leq i\leq N.$
For every $\{a_i\}_{i=1}^N$, we define 
\[
f(\bw):=\sum_{i=1}^{N} a_i \phi_{\bx_i}\left(\bw \right).
\]
Assume $f(\bw) = 0$ holds for $\bw\sim\mathcal{N}(0,\I_d).$
By definition, $f(\bw)$ is a continuous function on $\mathbb{R}^{d}$ and it means $f(\bw)\equiv 0$ and $\nabla f(\bw) \equiv \bzero_d$. 

Under $\bx_i \nparallel \bx_j$ for all $(i,j)$ pair, it holds $D_i \not\subset \cup_{j \neq i} D_j $, i.e., there exists $\bz \in D_i \setminus \cup_{j \neq i} D_j$ such that $\lag \bx_i,\bz\rag = 0$ and $\lag \bx_j,\bz\rag \neq 0$ for any $j\neq i.$
 Since $D_i$'s are closed sets, there exists $r_0>0$ such that for any $r\leq r_0$, ${\cal B}(\bz,r_0) \cap \left(\cup_{j \neq i} D_j\right) = \emptyset$. In other words, $\phi_{\bx_j}(\bw)$ is differentiable w.r.t. $\bw$ insides $ {\cal B}(\bz,r_0)$ for all $j \neq i$. For $i$, the ball ${\cal B}(\bz,r_0)$ contains two disjoint parts:
  \[
 {\cal B}^+(\bz,r):= \left\{\bw: \lag \bx_i,\bw \rag>0 \right\}\cap {\cal B}(\bz,r),~~{\cal B}^-(\bz,r):= \left\{\bw: \lag\bx_i,\bw\rag < 0 \right\}\cap {\cal B}(\bz,r)
  \]
 and $\bz$ is on the boundary of both ${\cal B}^+(\bz,r)$ and ${\cal B}^-(\bz,r).$

Therefore, there exist two sequences $\left\{\bw_{\ell}^+ \right\} \subseteq {\cal B}^+(\bz,r) $ and $\left\{\bw_{\ell}^- \right\}\subseteq {\cal B}^-(\bz,r)$ such that $\lim_{\ell \to \infty} \bw_{\ell}^+ = \lim_{\ell \to \infty} \bw_{\ell}^- = \bz$. Due to the continuous differentiability of $\phi_{\bx_j}(\bw)$ in ${\cal B}(\bz, r)$ for $j\neq i$, we have
  \begin{equation}
 \lim_{\ell \to \infty} \nabla \phi_{\bx_j}(\bw_{\ell}^+) = \lim_{\ell \to \infty} \nabla \phi_{\bx_j}(\bw_{\ell}^-) = \nabla \phi_{\bx_j}(\bz).
  \end{equation}
 For $i$, we note that $\phi_{\bx_i}(\bw)$ is not differentiable at $\bz\in D_i$ while the gradient exists on ${\cal B}^+(\bz,r)$ and ${\cal B}^-(\bz,r).$ It holds
  \begin{equation}
 \lim_{\ell\to \infty} \nabla \phi_{\bx_i}(\bw_{\ell}^+) = \lim_{\ell \to \infty} \bw_{\ell} = \bz,~~ \lim_{\ell\to \infty} \nabla \phi_{\bx_i}(\bw_{\ell}^-) = 0
  \end{equation}
 where 
  \[
  \phi_{\bx_i}(\bw) = 
  \begin{cases}
 \lag \bx_i, \bw\rag - \sqrt{\frac{2}{\pi}}, & \bw\in {\cal B}^+(\bz,r), \\
 - \sqrt{\frac{2}{\pi}}, & \bw\in {\cal B}^-(\bz,r).
  \end{cases}
  \]
 Therefore, 
  \begin{align*}
  0 =& \lim_{\ell \to \infty} \nabla f(\bw^+_{\ell}) - \lim_{\ell \to \infty} \nabla f(\bw^-_{\ell})\\
 =&\lim_{\ell \to \infty} \sum_{j \neq i}^{N}a_j (\nabla\phi_{\bx_j}(\bw^+_{\ell}) - \nabla\phi_{\bx_j}(\bw^-_{\ell})) + a_i\lim_{\ell \to \infty} \left(\nabla\phi_{\bx_i}(\bw^+_{\ell})-\nabla\phi_{\bx_i}(\bw^-_{\ell})\right) = a_i\bz.
  \end{align*}
Note that $\bz\neq 0$ and it implies $a_i = 0.$
\end{proof}


\begin{proof}[\bf Proof of Theorem~\ref{Thm:three_layer_main}]
By Theorem~\ref{thm:NC_general}(b), when $d_1 \geq N$, it suffices to have 
$\sigma_{\ReLU}(\BW_1^\top\BX) \in \mathbb{R}^{d_1 \times N}$ to be rank $N$ to induce neural collapse. Without loss of generality, We assume each $\bx_i$ is a unit vector in $\RR^d$. 
Here we define $\bvarphi_k$ as
\[
\bvarphi_k = \sigma_{\ReLU}(\BX^{\top}\bz_k) = 
\begin{bmatrix}
\sigma_{\ReLU}(\lag \bx_1,\bz_k\rag) \\
\vdots \\
\sigma_{\ReLU}(\lag \bx_N,\bz_k\rag)
\end{bmatrix},~~~\E \bvarphi_k = \sqrt{\frac{2}{\pi}}\bone_d,
\]
which is exactly the $k$-th row of $\sigma_{\ReLU}(\BW_1^{\top}\BX)$ and it is a sub-gaussian random vector.
To guarantee that $\sigma_{\ReLU}(\BW_1^{\top}\BX)$ is full rank, it suffices to have
\[
\left\| \frac{1}{d_1}\sum_{k=1}^{d_1}(\bvarphi_k -\E\bvarphi_k)(\bvarphi_k-\E \bvarphi_k)^{\top} - \BH_{\infty} \right\| < \lambda_{\min}(\BH_{\infty}).
\]
 The first step is directly implied by Lemma~\ref{lemma:tl_exp}, i.e. $\lambda_{\min}(\BH_{\infty})> 0$. 
 The estimation reduces to the covariance matrix estimation, which can be done by computing the sub-gaussian norm of the centered $\bvarphi_k$. For any $\bv \in \mathbb{S}^{N-1}$ and define 
\[
f(\bz):= \left\lag \bv,\sigma_{\ReLU}\left(\BX^\top\bz\right)-\sqrt{\frac{2}{\pi}}\bone_{N} \right\rag,
\]
which is a random variable depending on $\bz.$
Then for any $\bv$, it holds
\begin{equation}
\begin{aligned}
\left|f(\bz_1)-f(\bz_2) \right| \leq \left\|\sigma_{\ReLU}\left(\BX^\top\bz_1\right)-\sigma_{\ReLU}\left(\BX^\top\bz_2\right)\right\| \leq \left\|\BX\right\| \|\bz_1-\bz_2\|.
\end{aligned}
\end{equation}
Then by~\cite[Theorem 2.26]{W19}, we know that $f(\bz) - \E f(\bz)$ is a subgaussian random variable with 
\[
\E e^{\lambda \lag \bv, \sigma_{\ReLU}(\BX^{\top}\bz) - \E \sigma_{\ReLU}(\BX^{\top}\bz)\rag} \leq e^{\lambda^2\|\BX\|^2/2},~~\forall\bv\in \mathbb{S}^{N-1},~\lambda>0.
\] 
Therefore, the random vector $\sigma(\BX^{\top}\bz) - \E \sigma(\BX^{\top}\bz) $ has a sub-gaussian norm bounded by $\left\|\BX\right\|$.

By~\cite[Exercise 4.7.3]{V18} or~\cite[Theorem 6.5]{W19}, it holds that 
\[
 \left\|\frac{1}{d_1}\sum_k (\bvarphi_k - \E \bvarphi_k)(\bvarphi_k - \E \bvarphi_k)^{\top}-\BH_{\infty} \right\| \lesssim \|\BX\|^2 \sqrt{\frac{N}{d_1}} < \lambda_{\min}(\BH_{\infty})
\]
with probability at least $1-2e^{-N}$
provided that
$d_1 \gtrsim \|\BX\|^4 N\log N/\lambda^2_{\min}(\BH_{\infty}).$ 

Note that
\begin{align*}
& \frac{1}{d_1}\sum_k \bvarphi_k\bvarphi_k^{\top} - \E \bvarphi_k\bvarphi_k^{\top} = \frac{1}{d_1}\sum_k \bvarphi_k\bvarphi_k^{\top} -(\BH_{\infty} + \frac{2}{d_1\pi}\BJ_N) \\
& = \frac{1}{d_1}\sum_{k=1}^{d_1} (\bvarphi_k - \E \bvarphi_k)(\bvarphi_k - \E \bvarphi_k)^{\top}-\BH_{\infty} + \frac{1}{d_1}\sqrt{\frac{2}{\pi}} \left( \bone_N \sum_{k=1}^{d_1} (\bvarphi_k - \E\bvarphi_k)^{\top} + \sum_{k=1}^{d_1} (\bvarphi_k-\E\bvarphi_k)\bone_N^{\top} \right)
\end{align*}
where $\E\bvarphi_k = \sqrt{2/\pi}\bone_{d_1}.$ 
It suffices to estimate $\| d_1^{-1}\sum_k \bvarphi_k - \E\bvarphi_k\|$.
We know that
\[
\E e^{\lambda \sum_k \lag \bv, \sigma_{\ReLU}(\BX^{\top}\bz_k) - \E \sigma_{\ReLU}(\BX^{\top}\bz_k)\rag} \leq e^{\lambda^2d_1\|\BX\|^2/2},~~\forall\bv\in \mathbb{S}^{N-1},~\lambda>0
\]
and thus $d_1^{-1} \sum_k \lag \bv, \sigma_{\ReLU}(\BX^{\top}\bz_k) - \E \sigma_{\ReLU}(\BX^{\top}\bz_k)\rag$ is $\|\BX\|^2/d_1$-subgaussian. Thus
\begin{align*}
\Pr\left( \left\| \frac{1}{d_1}\sum_k \sigma_{\ReLU}(\BX^{\top}\bz_k) - \E\bvarphi_k\right\| \geq t\right) \leq 2\exp\left(-\frac{d_1t^2}{2\|\BX\|^2}\right).
\end{align*}
By picking $t = 2\|\BX\| \sqrt{d_1^{-1}\log N}$, we have
\begin{align*}
\left\| \frac{1}{d_1}\sum_k \bvarphi_k\bvarphi_k^{\top} - \E \bvarphi_k\bvarphi_k^{\top} \right\| \lesssim\|\BX\|^2 \sqrt{\frac{N}{d_1}} + \|\BX\| \sqrt{\frac{N\log N}{d_1}} \lesssim \lambda_{\min}(\BH_{\infty})
\end{align*}
under $d_1\gtrsim \|\BX\|^4 N\log N/\lambda^2_{\min}(\BH_{\infty}).$ Therefore, the smallest eigenvalue of $d_1^{-1}\sum_k \bvarphi_k\bvarphi_k^{\top}$ is at least $\lambda_{\min}(\BH_{\infty})/2$, implying that $\sigma_{\ReLU}(\BW_1^{\top}\BX)\in\RR^{d_1\times N}$ is of full column rank.
\end{proof}

\subsection{Two-layer neural network: best generalization under ${\cal NC}$}\label{ss:twonn_gen}

\subsubsection{Proof of Theorem~\ref{thm:nc_gen}(a)}
\begin{proof}[\bf Proof of Theorem~\ref{thm:nc_gen}(a)]

We consider the minimization for the term involving $\bbeta$, 
\begin{equation}\label{eq:genob}
\min_{\bbeta} \Pr\left({\bbeta}^\top\bmu+\sigma{\bbeta}^\top\bz \leq 0 \right)~~~\text{s.t.}~~\left(\bone_n \bmu^\top + \sigma \BZ_1\right) \bbeta = \bone_n,~~\left(-\bone_n \bmu^\top + \sigma \BZ_2\right) \bbeta \leq 0. 
\end{equation} 
The difficulty of optimization mainly arises from incorporating the inequality constraints. For now, we drop the inequality constraints and consider the following simplified version that only involves the equality constraints, then the minimization is equivalent to solving the following maximization program. 
\begin{equation}\label{eq:genoblow}                   
 \max_{\bbeta\in \mathbb{R}^d} \frac{\lag\bmu,\bbeta\rag}{\|\bbeta \|}~~~\text{s.t.}~~\left(\bone_n \bmu^\top + \sigma \BZ_1\right) \bbeta = \bone_n. 
\end{equation}
Note that for any $\bbeta$ satisfying $(\bone_n\bmu^{\top} + \sigma\BZ_1)\bbeta = \bone_n$, it holds that
\[
(\I_n - \BJ_n/n)\BZ_1\bbeta = 0,~~~ \left\lag \bmu + \frac{\sigma\BZ_1^{\top}\bone_n}{n}, \bbeta\right\rag = 1.
\] 
Therefore, we let $\BPhi$ be the orthonormal basis of the null space of $(\I_n - \BJ_n/n)\BZ_1$. Since $\BZ_1\in\RR^{n\times d}$ is a Gaussian random matrix, $\BPhi$ is a random matrix sampled from $G_{d,d-n+1}.$

Therefore, we have the following representation for $\bbeta$: 
\[
\bbeta(\bv) = \BPhi\left(\frac{\BPhi^\top\widehat{\bmu}}{\|\BPhi^\top\widehat{\bmu} \|^2} + \bv\right)
\]
where
\[
\widehat{\bmu} = \bmu + \frac{\sigma\BZ_1^\top \bone_n}{n},~~ \lag\bv,\BPhi^\top \widehat{\bmu} \rag = 0.
\]

Now we can see the maximization of~\eqref{eq:genoblow} is equivalent to maximizing
\begin{align*}
f(\bv):= \frac{\lag \bmu,\bbeta(\bv)\rag}{\|\bbeta(\bv)\|} = \frac{\lag\BPhi^\top \bmu, \BPhi^\top\widehat{\bmu}/\|\BPhi^\top \widehat{\bmu} \|^2 + \bv \rag}{\sqrt{1/\|\BPhi^\top \widehat{\bmu}\|^2 + \|\bv\|^2}}
\end{align*}
over $\bv$ such that $\lag \bv, \BPhi^{\top}\widehat{\bmu}\rag=0.$

We apply Lemma~\ref{lemma:featc_opt} to $-f(\bv)$ with $\bv_1:= \BPhi^\top \widehat{\bmu}/ \| \BPhi^\top \widehat{\bmu}\|^2 $ and $\bv_2:= -\BPhi^\top \bmu$. Then we obtain
\begin{equation}\label{eq:genln0}
 f^* = \max_{\lag \bv,\BPhi^\top\widehat{\bmu}\rag=0} f(\bv)
 =
  \begin{cases} 
  \|\BPhi^\top \bmu \|\cdot|\sin\widetilde{\theta}|,~~ & \text{if} ~~ \cos\widetilde{\theta} \geq 0, \\ 
  \|\BPhi^\top\bmu\|, ~~ &\text{if}~~ \cos\widetilde{\theta} < 0,
  \end{cases}
\end{equation} 
where $\widetilde{\theta}$ is the angle between $\BPhi^{\top}\bmu$ and $-\BPhi^{\top}\widehat{\bmu}.$ 

Now we aim to identify a regime for $\sigma$ such that the global maximum of~\eqref{eq:genln0} matches the minimum of~\eqref{eq:genob}, i.e. the inequality constraints hold with high probability. 
It occurs if 
\[
\cos\widetilde{\theta} < 0,~~~ -\lag \bmu, \bbeta^*\rag \bone_n + \sigma\BZ_2\bbeta^* \leq 0
\]
where $\lag \bmu, \bbeta^*\rag/\|\bbeta^*\| = \|\BPhi^{\top}\bmu\|.$
For the first constraint, it holds
\[
-\lag \BPhi^\top \bmu,\BPhi^\top \widehat{\bmu} \rag = -\|\BPhi^{\top}\bmu\|^2 - \frac{\sigma}{n} \lag \BZ_1^{\top}\bone_n, \BPhi^{\top}\bmu\rag \sim\mathcal{N}( -\|\BPhi^{\top}\bmu\|^2, \sigma^2 \|\BPhi^{\top}\bmu\|^2/n).
\]
Therefore, 
\[
\Pr(\cos\widetilde{\theta} < 0) = \Pr(Z < \sqrt{n} \sigma^{-1} \|\BPhi^{\top}\bmu\| ),
\]
which means $\cos\widetilde{\theta} < 0$ holds with probability at least $1-O(n^{-1})$ if 
\begin{equation}\label{eq:genln1}
\sigma \sqrt{\log n} < \sqrt{n}\|\BPhi^{\top}\bmu\|.
\end{equation}
For the second inequality to hold with probability at least $1-O(n^{-1})$, we have
\[
\Pr(-\lag \bmu,\bbeta^* \rag + \sigma \lag \bz, \beta^* \rag < 0) = \Pr(Z < \sigma^{-1} \|\BPhi^{\top}\bmu\| ),
\]
then we need
\begin{equation}\label{eq:genln2}
2 \sigma \sqrt{\log n} < \|\BPhi^{\top}\bmu\|,
\end{equation}
by taking union bound over the $n$ samples.
As~\eqref{eq:genln2} is more strict than~\eqref{eq:genln1}, we have the best possible generalization error bound is attained with probability at least $1-O(n^{-1})$ if
\begin{equation}\label{thmpeq:gen_nc_a_1}
2\sigma \sqrt{\log n} < (1-\eps) \sqrt{\frac{d-n+1}{d}}\|\bmu\|.
\end{equation}
where we construct the lower bound (the RHS) for $ \|\BPhi^{\top}\bmu\|$ by applying Johnson-Lindenstrauss lemma. Under condition~\eqref{thmpeq:gen_nc_a_1}, the best misclassification error is upper bounded by
\begin{align*}
\Pr\left(Z + \frac{f^*}{\sigma} < 0\right) & \leq \Pr\left(Z > \frac{1-\eps}{\sigma}\sqrt{\frac{d-n+1}{d}}\|\bmu\|\right) \\
& \leq \frac{1}{\sqrt{2\pi}} \exp\left(-\frac{(1-\eps)^2\|\bmu\|^2}{2\sigma^2}\frac{d-n+1}{d}\right) \leq n^{-2}
\end{align*}
where $f^* \geq (1-\eps)\|\bmu\|\sqrt{(d-n+1)/d} $. 

\end{proof}


\subsubsection{Proof of Theorem~\ref{thm:nc_gen}(b): an upper bound on~\eqref{eq:genob} with $d > 2n$}

\begin{proof}[\bf Proof of Theorem~\ref{thm:nc_gen}(b)]
Consider
\[
\BX_1^{\top} = \bone_n\bmu^{\top} + \sigma\BZ_1,~~~\BX_2^{\top} = -\bone_n\bmu^{\top} + \sigma\BZ_2.
\]
We assume $\bmu =\|\bmu\| \be_1$ without loss of generality, and we will need to provide a lower bound of $\min_{\bbeta} \Pr_{z\sim\mathcal{N}(0,1)}( \lag \bmu,\bbeta\rag + \sigma\|\bbeta\|z \leq 0 )$ subject to the constraints imposed by ${\cal NC}$. 
This lower bound leads us to consider program~\eqref{eq:genob} again. The difference here is we assume $d > 2n$, i.e., the neural collapse occurs with probability 1. Our aim is to establish the best possible generalization bound v.s. $\sigma/\|\bmu\|.$

The idea to minimize~\eqref{eq:genob} follows from two steps.
We first consider
\[
F(c,\bgamma) : = \max_{\bbeta\in \mathbb{R}^d} \frac{ \lag \bmu,\bbeta\rag}{\|\bbeta \|}~~~\text{s.t.}~~ \left(\bone_n \bmu^\top + \sigma \BZ_1\right) \bbeta = \bone_n, ~~ \left(-\bone_n \bmu^\top + \sigma \BZ_2\right) \bbeta = -\bgamma,~~\lag \bmu,\bbeta\rag = c
\]
for $c>0$ and $\bgamma \geq 0.$ Note that for any $\bgamma$ and $c > 0$, there exist feasible $\bbeta$'s.
To obtain the maximum of $\lag \bmu,\bbeta\rag/\|\bbeta\|$ is equivalent to minimizing $\|\bbeta\|$. Once we have that, the second step is to maximize over $c$ and $\bgamma$.

Now we compute $F(c,\bgamma)$ by minimizing $\|\bbeta\|$ subject to the constraints:
\begin{align*}
\lag \bmu,\bbeta\rag\bone_n + \sigma\BZ_1\bbeta = \bone_n,~~~ -\lag\bmu, \bbeta \rag\bone_n + \sigma \BZ_2\bbeta = -\bgamma,~~~\lag \bmu,\bbeta\rag = c.
\end{align*}
Note that the third constraint above implies $\beta_1 = c \be_1/ \|\bmu\|.$
Substituting it into the first two equalities gives
\[
\sigma\BZ_1\bbeta = -c\bone_n,~~~ \sigma \BZ_2\bbeta = -\bgamma + c\bone_n,
\]
and then
\[
\BZ_{(-1)} \bbeta_{(-1)} = \frac{1}{\sigma}
\begin{bmatrix}
(1-c)\bone_n \\
c\bone_n - \bgamma
\end{bmatrix} - \frac{c\BZ\be_1}{\|\bmu\|}
\]
where $\BZ_{(-1)}$ is a Gaussian matrix of size $2n\times (d-1)$ that excludes the first column of $\BZ$, and $\bbeta_{(-1)}$ is the same as $\bbeta$ after removing the first entry in $\bbeta$.
The minimum norm solution to $\bbeta_{(-1)}$ is given by
\[
\bbeta_{(-1)} = \BZ_{(-1)}^{\top} (\BZ_{(-1)}\BZ_{(-1)}^{\top})^{-1}\left(\frac{1}{\sigma}
\begin{bmatrix}
(1-c)\bone_n \\
c\bone_n - \bgamma
\end{bmatrix} - \frac{c\BZ\be_1}{\|\bmu\|} \right),~~\bbeta_1 = \frac{c}{ \|\bmu\|}
\]
where $d-1\geq 2n$ ensures the invertibility of $\BZ_{(-1)}$
As a result, we have
\[
F(c,\bgamma) = \frac{c}{\sqrt{\left\|\BZ_{(-1)}^{\top} (\BZ_{(-1)}\BZ_{(-1)}^{\top})^{-1}\left(\frac{1}{\sigma}
\begin{bmatrix}
(1-c)\bone_n \\
c\bone_n - \bgamma
\end{bmatrix} - \frac{c\BZ\be_1}{\|\bmu\|} \right)\right\|^2 + \frac{c^2}{\|\bmu\|^2} }}.
\]
Maximizing $F(c,\bgamma)$ is equivalent to minimizing $1/F^2(c,\bgamma)$:
\begin{align*}
\frac{1}{F^2(c,\bgamma)}& = \frac{1}{c^2} \left\|\BZ_{(-1)}^{\top} (\BZ_{(-1)}\BZ_{(-1)}^{\top})^{-1} \left(\frac{1}{\sigma}
\begin{bmatrix}
(1-c)\bone_n \\
c\bone_n - \bgamma
\end{bmatrix} - \frac{c\BZ\be_1}{\|\bmu\|} \right) \right\|^2 + \frac{1}{\|\bmu\|^2} \\
& = \left\|\BZ_{(-1)}^{\top} (\BZ_{(-1)}\BZ_{(-1)}^{\top})^{-1}\left(\frac{1}{\sigma}
\begin{bmatrix}
(1/c-1)\bone_n \\
\bone_n - \bgamma/c
\end{bmatrix} - \frac{\BZ\be_1}{\|\bmu\|} \right) \right\|^2 + \frac{1}{\|\bmu\|^2}.
\end{align*}

For $\BZ_{(-1)}\in\RR^{2n\times (d-1)}$ Gaussian random matrix, we have
\[
\frac{d}{2}\I\preceq \BZ_{(-1)}\BZ_{(-1)}^{\top}
\] 
holds with high probability provided that $d \geq 2n \log n.$
Let
\[
G(c,\bgamma) = \left\|\frac{1}{\sigma}
\begin{bmatrix}
(1/c-1)\bone_n \\
\bone_n - \bgamma/c
\end{bmatrix} - \frac{\BZ\be_1}{\|\bmu\|} \right\|^2 =
 \left\| \frac{1-c}{c\sigma}\bone_n + \frac{\BZ_1\be_1}{\|\bmu\|} \right\|^2 
+ \left\| \frac{1}{\sigma} \left(\bone_n - \frac{\bgamma}{c}\right) - \frac{\BZ_2\be_1}{\|\bmu\|}\right\|^2 
\]
and then
\[
\frac{G(c,\bgamma)}{2d} + \frac{1}{\|\bmu\|^2}\leq \frac{1}{F^2(c,\bgamma)}.
\]

We proceed to compute the minimum of $G(c,\bgamma)$. For any positive $c>0$, the minimizer of the second term is attained at 
\[
\bgamma = c\left[\bone_n-\frac{\sigma\BZ_2\be_1}{\|\bmu\|}\right]_{+}
\]
with minimum value
\[
\min_{\bgamma\geq 0}\left\| \frac{1}{\sigma} \left(\bone_n - \frac{\bgamma}{c}\right) - \frac{\BZ_2\be_1}{\|\bmu\|}\right\|^2 = \left\| \left[\frac{\BZ_2\be_1}{\|\bmu\|}-\frac{1}{\sigma}\bone_n\right]_{+}\right\|^2,
\]
which is independent of $c$. So we are left to minimize the first term, which is a quadratic function in $1/c$, 
\begin{equation}
  \begin{aligned}
\underset{c>0}{\argmin}~\frac{1}{\sigma^2} \left\|\frac{1}{c}\bone_n+\left[\frac{\sigma\BZ_1\be_1}{\|\bmu\|}-\bone_n\right] \right\|^2 
& =\underset{c>0}{\argmin}~\frac{n}{c^2} + \frac{2}{c}\left\lag \bone_n, \frac{\sigma\BZ_1\be_1}{\|\bmu\|} - \bone_n \right\rag \\
&= \begin{cases} \frac{n\|\bmu \|}{n\|\bmu\|-\sigma\bone_n^\top\BZ_1\be_1},~~&\text{if}~~ n\|\bmu\|-\sigma\bone_n^\top\BZ_1\be_1>0, \\
 +\infty,~~&\text{else},\end{cases}
  \end{aligned}
\end{equation}
and thus
\[
\min_{c,\bgamma} G(c,\bgamma) = \frac{1}{\sigma^2} \left\| \left[ s\BZ_2\be_1- \bone_n\right]_+\right\|^2 + \begin{cases}\frac{1}{\|\bmu\|^2} \left\|\left(\BI_n-\frac{\bone_n\bone_n^\top}{n}\right)\BZ_1\be_1 \right\|^2, ~~&\text{if}~~ n\|\bmu \| - \sigma\bone_n^\top\BZ_1\be_1>0, \\
 \frac{1}{\sigma^2}\left\| -\bone_n + s\BZ_1\be_1 \right\|^2, ~~&\text{else}\end{cases}
\]
where $s = \sigma/\|\bmu\|.$ Now we provide a high probability value bound for the second term in two cases respectively.

\noindent \textbf{Case $1$: $n\|\bmu\|-\sigma\bone_n^\top\BZ_1\be_1 \leq 0$:}
The term is a sum of exponential random variable: $\E |1 + sZ|^2 = 1 + s^2$ for $Z\sim\mathcal{N}(0,1)$ and
\begin{align*}
\E(|1 + sZ|^2 - 1-s^2)^2 & = s^2\E(2Z + s Z^2 - s)^2 = s^2(4 + 2s^2 ).
\end{align*}
Then we have
\begin{equation}\label{eq:gen_b_proof1}
\left|\|s\BZ_1\be_1 + \bone_n\|^2 - n(1+s^2)\right| \leq Cs \sqrt{(4+2s^2)n\log n}
\end{equation}
holds with probability at least $1-O({n^{-2}})$.

\noindent \textbf{Case $2$: $n\|\bmu\|-\sigma\bone_n^\top\BZ_1\be_1 > 0$:}
In this case, we want to obtain concentration bound for quadratic form $\bz^\top \BC_n \bz$ where $\bz \sim \mathcal{N}(0,\BI_n)$. Note that $\bz^{\top}\BC_n\bz\sim\chi^2_{n-1}$, and then we have 
\begin{equation}\label{eq:gen_b_proof2}
\left\vert \bz^\top \BC_n \bz-\E\bz^\top \BC_n \bz\right\vert > c\sqrt{n \log{n}} 
\end{equation}
with probability as least $1-O(n^{-1})$ where we use the fact that $\|\BC_n\|_{op}=1$ and we have $\E\bz^\top \BC_n \bz = \tr(\BC_n) = n-1$.

For the first term, we notice it is also a sum of exponential random variables, so we have 
\begin{equation}\label{eq:gen_b_proof3}
\left| \| [-\bone_n + s\BZ_2\be_1]_+\|^2 -n\E[-1 + sZ]_+^2 \right| \leq Cs \sqrt{(4+2s^2)n\log n}
\end{equation}
holds with probability at least $1-O(n^{-2}).$ 

In expectation, it holds that
\begin{align*}
\E [-1 + s Z]_+^2 & = \frac{1}{\sqrt{2\pi s^2}} \int_0^{\infty} x^2\exp\left(-\frac{(x+1)^2}{2s^2}\right)\diff x \\
& = \frac{1}{\sqrt{2\pi s^2}} \int_1^{\infty} (x-1)^2 \exp\left(-\frac{x^2}{2s^2}\right)\diff x \\
& = \frac{1}{\sqrt{2\pi s^2}}\left( (s^2+1)\int_1^{\infty}e^{-\frac{x^2}{2s^2}}\diff x - s^2 e^{-\frac{1}{2s^2}} \right).
\end{align*}

Note that
\begin{align*}
\int_1^{\infty} x\exp(-\frac{x^2}{2s^2})\diff x & =-s^2 \int_1^{\infty} \diff e^{-\frac{x^2}{2s^2}} = s^2 \exp(-\frac{1}{2s^2}), \\
\int_1^{\infty} x^2 \exp(-\frac{x^2}{2s^2})\diff x & = -s^2 \int_1^{\infty} x \diff e^{-\frac{x^2}{2s^2}} = s^2 e^{-\frac{1}{2s^2}} + s^2 \int_1^{\infty} e^{-\frac{x^2}{2s^2}}\diff x, \\
\int_1^{\infty} \exp(-\frac{x^2}{2s^2})\diff x & \leq \int_1^{\infty} x \exp(-\frac{x^2}{2s^2})\diff x = s^2 e^{-\frac{1}{2s^2}}.
\end{align*}

As a result, we have
\[
\E[-1+sZ]_+^2 \leq \frac{s^3e^{-\frac{1}{2s^2}}}{\sqrt{2\pi}}.
\]
Hence, collecting all the concentration bounds~\eqref{eq:gen_b_proof1},\eqref{eq:gen_b_proof2},\eqref{eq:gen_b_proof3} we have derived above, it holds with at least probability $1-O(n^{-1})$,
\[
\min G(c,\bgamma) \geq \begin{cases} \frac{n}{\sigma^2}\left(\frac{s^3e^{-\frac{1}{2s^2}}}{\sqrt{2\pi}}+ s^2 -(c_1's^2+c_2's)\sqrt{\frac{\log n}{n}}\right), &\text{if}~~s \bone_n^\top\BZ_1\be_1< n,\\ \frac{n}{\sigma^2} \left(\frac{s^3e^{-\frac{1}{2s^2}}}{\sqrt{2\pi}} + s^2+1-(c'_3s^2+c_4's)\sqrt{\frac{\log n}{n}}\right),~~&\text{else}, \end{cases}
\]
for some positive constants $c_1',c_2',c_3'$, and $c_4' $. Recall that we have
\[
F^{*} = \max F(c,\bgamma) \leq \left( \frac{\min G(c,\bgamma)}{2d}+ \frac{1}{\|\bmu\|^2}\right)^{-1/2}.
\]
As a result, the best possible misclassification error is lower bounded by:
\begin{equation}
  \begin{aligned}
 \Pr\left( \frac{F^*}{\sigma} + Z < 0\right) \geq \Pr\left(Z>\left( \frac{n }{2d}\left(\frac{s^3 e^{-\frac{1}{2s^2} }}{\sqrt{2\pi}} + s^2+1-(c_1s^2+c_2s)\sqrt{\frac{\log n}{n}}\right) + s^2\right)^{-1/2}\right)
  \end{aligned}
\end{equation}
where $c_1>0$ and $c_2>0$.
\end{proof}

\section{Appendix}
\begin{lemma}\label{lemma:featc_opt}
 Let $\bv_1$ and $\bv_2 \in \mathbb{R}^d$ be two arbitrary vectors. Then it holds
\[
 \min_{\lag\bv,\bv_1\rag = 0}h(\bv) : = \frac{\lag \bv_1+\bv, \bv_2 \rag}{\sqrt{\|\bv_1\|^2+\|\bv\|^2}} = 
  \begin{cases} 
 -\sqrt{ \|\bv_2\|^2 - |\lag \bv_1,\bv_2\rag|^2/\|\bv_1\|^2 }, ~~ &\text{if} ~~ \lag \bv_1,\bv_2\rag \geq 0, \\ 
 -\|\bv_2\|, ~~ &\text{if}~~ \lag \bv_1,\bv_2\rag < 0.
  \end{cases}
\]
  \end{lemma}
  \begin{proof} 
 Note that $\bv$ is in a linear subspace, and thus the global minimum is certainly negative. Let $c = \lag \bv_2, \bv_1 + \bv\rag < 0$ be fixed and then $h(\bv)$ is increasing w.r.t. $\|\bv\|$. As a result, we need to search for a $\bv$ with the minimum norm subject to 
\[
\lag \bv, \bv_1\rag = 0,~~\lag \bv,\bv_2\rag = c- \lag \bv_1,\bv_2\rag.
\] 
Let $\BP = \I - \bv_1\bv_1^{\top}/\|\bv_1\|^2$ be the projection matrix onto the complement of span($\bv_1$)
The minimum is given by 
\[
\bv = (c - \lag \bv_1,\bv_2\rag)\frac{\BP\bv_2 }{\|\BP\bv_2\|^2}.
\]
Therefore, the minimization of $h(\bv)$ is reduced to minimizing $h(c)$ for $c<0$:
  \begin{align*}
 h(c) & := \frac{c}{ \sqrt{\|\bv_1\|_2^2 + (c - \lag \bv_1,\bv_2\rag)^2/\|\BP\bv_2\|^2} } \\
  & = \frac{c\|\BP\bv_2\| }{ \sqrt{\|\bv_1\|^2\|\BP\bv_2\|^2 + (c - \lag \bv_1,\bv_2\rag)^2} } \\
  & = \frac{-\|\BP\bv_2\|}{ \sqrt{\left(\|\bv_1\|^2\|\BP\bv_2\|^2 + \lag \bv_1,\bv_2\rag^2\right)/c^2 + 2\lag \bv_1,\bv_2\rag/|c| + 
  1} }.
  \end{align*}
If $\lag \bv_1,\bv_2\rag \geq 0$, then function $h(c)$ satisfies
  \[
 \inf_{c<0} h(c) = -\|\BP\bv_2\| = -\sqrt{ \|\bv_2\|^2 - \frac{|\lag \bv_1,\bv_2\rag|^2}{\|\bv_1\|^2} }.
  \]
If $\lag \bv_1,\bv_2\rag < 0$, then the denominator is a quadratic function of $1/|c|$. The minimum is achieved at
  \[
 \frac{1}{c} = \frac{\lag \bv_1,\bv_2\rag}{ \|\bv_1\|^2\|\BP\bv_2\|^2 + \lag \bv_1,\bv_2\rag^2 } = \frac{\lag \bv_1,\bv_2\rag}{\| \bv_1\|^2\| \bv_2\|^2 }
  \] 
and then $\min_{c<0} h(c) = -\|\bv_2\|$.
  \end{proof}

\begin{lemma}[\textbf{Johnson-Lindenstrauss lemma for angles}]\label{lemma:JLangle}
~Let $\BPhi\in\mathbb{R}^{n \times m}$ be a random orthogonal matrix with $m > n$. For any two unit vectors $\bv_1$ and $\bv_2 \in \mathbb{R}^n$, we denote the angle between $\bv_1 $ and $\bv_2$ and $\BPhi^{\top}\bv_1$ and $\BV^{\top}\bv_2$ by $\theta $ and $\widetilde{\theta}$ respectively. Then with probability at least $1-4\exp(-c\eps^2m)$,
  \[
|\cos\widetilde{\theta} - \cos\theta| \leq \frac{4\eps}{(1-\eps)^2}.
  \]

\end{lemma}

\begin{proof}
 Without loss of generality, we assume $\|\bv_1\|_2=\|\bv_2\|_2=1$. By JL lemma, we have with probability at least $1-2\exp(-c\eps^2m)$, for $k=1,2$,
  \begin{equation}\label{eq:JLangle2}
 (1-\epsilon)\sqrt{\frac{m}{n}}\|\bv_k\|_2 \leq \|\BPhi^{\top}\bv_k\|_2 \leq (1+\epsilon)\sqrt{\frac{m}{n}}\|\bv_k\|_2.
  \end{equation}
 By applying union bound over $\|\BPhi^{\top}\bv_k\|_2, (k=1,2)$, we have with probability at least $1-4\exp(-c\eps^2m)$, the following bounds hold. We proceed to estimate $\cos\widetilde{\theta}:$
\begin{align*}
\cos\widetilde{\theta} & = \frac{\|\BPhi^{\top}(\bv_1+\bv_2)\|^2 -\|\BPhi^{\top}(\bv_1-\bv_2)\|^2 }{4\|\BPhi^{\top}\bv_1\|\|\BPhi^{\top}\bv_2\|} \\
& \leq \frac{(1+\eps)^2(1+\cos\theta)- (1-\eps)^2(1-\cos\theta)}{ 2(1-\eps)^2 } = \frac{2\eps + (1+\eps^2)\cos\theta}{(1-\eps)^2}.
\end{align*}
As a result, it holds
\begin{align*}
\cos\widetilde{\theta} - \cos\theta & \leq \frac{2\eps + (1+\eps^2)\cos\theta - (1-\eps)^2\cos\theta}{(1-\eps)^2} = \frac{2\eps +2\eps\cos\theta}{(1-\eps)^2}.
\end{align*}
Similarly, for the lower bound, we have
\begin{align*}
\cos\widetilde{\theta} & \geq \frac{\|\BPhi^{\top}(\bv_1+\bv_2)\|^2 -\|\BPhi^{\top}(\bv_1-\bv_2)\|^2 }{4\|\BPhi^{\top}\bv_1\|\|\BPhi^{\top}\bv_2\|} \\
& \geq \frac{(1-\eps)^2(1+\cos\theta)- (1+\eps)^2(1-\cos\theta)}{ 2(1+\eps)^2 } = -\frac{2\eps }{(1+\eps)^2} +\cos\theta.
\end{align*}
Therefore, for $\cos\theta>0$, we have
\[
-\frac{2\eps}{(1+\eps)^2} \leq \cos\widetilde{\theta} - \cos\theta \leq \frac{4\eps}{(1-\eps)^2} \Longleftrightarrow |\cos\widetilde{\theta} - \cos\theta| \leq \frac{4\eps}{(1-\eps)^2}.
\]
For $\cos \theta < 0$, we consider $-\bv_1$ and $\bv_2$ instead, and the same bound holds.
\end{proof}

\begin{lemma}[\bf Johnson-Lindenstrauss lemma for singular values]\label{lemma:JLbasis}
Let $\BV \in \mathbb{R}^{K \times d}$ and $\BPhi\in\RR^{d\times m}$ is a random subspace sampled uniformly from $G_{d,m}\subset \RR^{d\times m}$ with $m<d$ where $G_{d,m}$ stands for the Grassmannian consisting of all $m$-dimensional subspaces in $\RR^d$, then with probability at least $1-K^2\exp(-c\eps^2m)$ for some constant $c>0$, the following inequalities hold,
\begin{equation}
  \begin{aligned}
    \sigma_{\max}(\BV\BPhi) &\leq \sqrt{\frac{m}{d} (1+\eps^2+ 2K\eps)} \sigma_{\max}\left(\BV\right), \\
    \sigma_{\min}(\BV\BPhi) &\geq \sqrt{\frac{m}{d}\left(1-\eps^2-2K\eps\right)} \sigma_{\min}\left(\BV\right),
  \end{aligned}
\end{equation}
where $ \sigma_{\max}(\BV)$ and $ \sigma_{\min}(\BV)$ denote the largest and smallest singular value of matrix $\BV$. In other words, under $m > c^{-1}\eps^{-2} \log (K^2 d)$, the probability of success is at least $1-O(d^{-1})$.
\end{lemma}
\begin{proof}[\bf Proof of Lemma~\ref{lemma:JLbasis}]
The idea is to apply the Johnson-Lindenstrauss lemma to bound the quadratic form $\ba^\top \left(\BV\BPhi\BPhi^{\top}\BV^{\top}\right)\ba$ for $\ba \in \mathbb{S}^{K-1}$ where $\BPhi\in\RR^{d\times m}$ is a random orthogonal matrix with $\BPhi^{\top}\BPhi = \I_m.$ Without loss of generality, we can assume $\BV\BV^{\top} = \I_K$, as we can perform the SVD on $\BV$ and only take care of the orthogonal parts.

Let $\BV^{\top} = [\bv_1,\cdots,\bv_K]\in\RR^{d\times K}$ with orthogonal columns. Then by applying Johnson-Lindenstrauss lemma on $\bv_i$, see~\cite[Lemma 5.3.2]{V18}, we have with probability at least $1-K^2\exp(-c\eps^{2}m)$,
\begin{equation}
  \begin{aligned}\label{eq2LemmaJLob}
(1-\eps)\sqrt{\frac{m}{d}} &\leq \|\BPhi^{\top}\bv_i \| \leq (1+\eps)\sqrt{\frac{m}{d}}, \\
\sqrt{2}(1-\eps)\sqrt{\frac{m}{d}} &\leq \|\BPhi^{\top}(\bv_i+\bv_j) \| \leq \sqrt{2}(1+\eps)\sqrt{\frac{m}{d}}, \\ 
  \end{aligned}
\end{equation}
where we use the fact that $\|\bv_i\|=1$ and $\|\bv_i \pm \bv_j\|=\sqrt{2}~(i\neq j)$. We have
\begin{equation}\label{eq3LemmaJLob}
\begin{aligned}
 \ba^\top\left(\BV\BPhi\BPhi^{\top}\BV^{\top}\right)\ba &= \sum_{i=1}^K a_i^2\|\BPhi^{\top}\bv_i\|^2 + \sum_{i=1}^K \sum_{j\neq i}^K a_ia_j\lag \BPhi^{\top}\bv_i, \BPhi^{\top}\bv_j \rag\\
&=\sum_{i=1}^K a_i^2\|\BPhi^{\top}\bv_i\|^2 + \sum_{i=1}^K\sum_{j\neq i}^K \frac{a_ia_j}{2} \left(\|\BPhi^{\top}(\bv_i+\bv_j)\|^2 - \|\BPhi^{\top}\bv_i\|^2 - \|\BPhi^{\top}\bv_j\|^2\right).
\end{aligned}
\end{equation}
Applying~\eqref{eq2LemmaJLob}, we obtain
\begin{equation}\label{eq4LemmaJLob}
  \|\BPhi^{\top}(\bv_i+\bv_j)\|^2 - \|\BPhi^{\top}\bv_i\|^2 - \|\BPhi^{\top}\bv_j\|^2 \leq \frac{2m}{d}\left( 1+\eps^2 - (1-\eps)^2\right) \leq \frac{4m\eps }{d}.
\end{equation}
Therefore by~\eqref{eq2LemmaJLob},~\eqref{eq3LemmaJLob} and~\eqref{eq4LemmaJLob}, it holds
\begin{equation}
  \begin{aligned}
 \ba^\top \left(\BV\BPhi\BPhi^{\top}\BV^{\top}\right)\ba 
  &\leq \frac{m}{d}\sum_{i=1}^K a_i^2(1+\eps)^2 + \frac{2\eps m}{d}\sum_{i=1}^K\sum_{j\neq i}^K |a_ia_j| \\
  &\leq \frac{m}{d} ((1+\eps)^2+ 2(K-1)\eps) = \frac{m}{d}(1+\eps^2 + 2K\eps)
  \end{aligned}
\end{equation}
where $\sum_{i\neq j} |a_ia_j| \leq K-1.$
Similarly,
\begin{equation}
  \begin{aligned}
 \ba^\top \left(\BV\BPhi\BPhi^{\top}\BV^{\top}\right)\ba 
  &\geq \frac{m}{d}\sum_{i=1}^K a_i^2(1-\eps)^2 -\frac{2\eps m}{d}\sum_{i=1}^K\sum_{j\neq i}^K |a_ia_j| \\
  &\geq \frac{m}{d}\left((1-\eps)^2-2(K-1)\eps\right) = \frac{m}{d}(1+\eps^2 -2K\eps).
  \end{aligned}
\end{equation}
\end{proof}

\begin{proposition}\label{Thm:Gordontess}
Let $\BZ\in\RR^{n\times d}$ be a standard Gaussian random matrix with i.i.d. entries. Then it holds
\begin{align*}
& \E\min_{\bs\in\mathbb{S}_+^{n-1}}\max_{\bu\in\mathbb{S}^{d-1}} \lag \bs, \BZ\bu\rag \geq \E_{\bg\sim\mathcal{N}(0,\I_n),\bh\sim\mathcal{N}(0,\I_d)} \left[ \|\bh\| - \|\bg_+\|\right] \sim \sqrt{d}- \sqrt{n/2}
\end{align*}
where $\bs\in\RR^n$ and $\bu\in\RR^d$, and $\bg_+$ is the positive part of $\bg$. Here
\[
 \E_{\bg\sim\mathcal{N}(0,\I_n)} \|\bg_+\| \leq \sqrt{ \E_{\bg\sim\mathcal{N}(0,\I_n)} \|\bg_+\|^2}
\]
and
\begin{align*}
\E_{\bh\sim\mathcal{N}(0,\I_d)} \|\bh\| & = \frac{\sqrt{2}\Gamma(\frac{d+1}{2})}{\Gamma(\frac{d}{2})}\sim 
\sqrt{\frac{d-1}{d-2}} \left( 1 + \frac{1}{d-2}\right)^{\frac{d-2}{2}} \frac{\sqrt{d-1}}{\sqrt{e}} \sim \sqrt{d}.
\end{align*}
\end{proposition}
The core idea of the proof relies on Gordon's inequality. 

\begin{theorem}[\bf Gordon's inequality]
Suppose $X_{u,t}$ and $Y_{u,t}$ are two Gaussian processes indexed by $(u,t)$. Assume that
\begin{align*}
\E (X_{u,t} - X_{u,s})^2 & = \E(Y_{u,t} - Y_{u,s})^2,~~\forall u,t,s, \\
\E (X_{u,t} - X_{v,s})^2 & \geq \E(Y_{u,t} - Y_{v,s})^2,~~\forall u\neq v,~\text{and }t,s,
\end{align*}
then 
\[
\E\inf_u \sup_t X_{u,t} \leq \E\inf_u \sup_t Y_{u,t}.
\]

\end{theorem}

\begin{proof}[\bf Proof of Proposition~\ref{Thm:Gordontess}]
Consider two Gaussian processes:
\[
X_{s,u} = \lag \bs, \BZ\bu\rag,~~~
Y_{s,u} = \lag \bs,\bg\rag + \lag \bh,\bu\rag
\]
where $\bg\sim\mathcal{N}(0,\I_n)$ and $\bh\sim\mathcal{N}(0,\I_d)$, and all
Then
\begin{align*}
\E |X_{s,u} - X_{t,v}|^2 & = \|\bu\bs^{\top} - \bv\bt^{\top}\|^2_F = 2 - 2\lag \bu,\bv\rag\lag \bs,\bt\rag \\
& \leq \|\bs - \bt\|^2 + \|\bu - \bv\|^2 = \E\|Y_{s,u} - Y_{t,v}\|^2.
\end{align*}
For $\bs = \bt$, it holds 
\[
\E \|X_{s,u} - X_{s,v}\|^2 = \|\bu- \bv\|^2 = \E\| Y_{s,u} - Y_{s,v} \|^2.
\]
Note that
\[
\E \min_{\bs\in\mathbb{S}_+^{n-1}}\max_{\bu\in\mathbb{S}^{d-1}} \{\lag \bs,\bg\rag + \lag \bh,\bu\rag\} = \E [-\|\bg_+\| +\|\bh\| ].
\]
By Gordon's inequality, it holds that
\begin{align*}
\E\min_{\bs\in\mathbb{S}_+^{n-1}}\max_{\bu\in\mathbb{S}^{d-1}} \lag \bs, \BZ\bu\rag & = \E \min_{\bs\in\mathbb{S}_+^{n-1}} \max_{\bu\in\mathbb{S}^{d-1}} X_{s,u} \\
& \geq \E \min_{\bs\in\mathbb{S}_+^{n-1}} \max_{\bu\in\mathbb{S}^{d-1}} Y_{s,u} = \E \left[ \|\bh\| - \|\bg_+\|\right]
\end{align*}
where $\bh\sim\mathcal{N}(0,\I_d)$ and $\bg\sim\mathcal{N}(0,\I_n).$
\end{proof}



\begin{lemma}\label{lemma:phi_2_concen}
 Let $\BZ \in \mathbb{R}^{n \times d}$ be a standard Gaussian random matrix. Define 
  \[
 f(\BZ):=\min_{\bs \in \mathbb{S}^{n-1}_+} \max_{\bu \in\mathbb{S}^{d-1}} \lag \bs,\BZ \bu \rag,
  \]
 we have
  \begin{equation}
 \Pr( |f(\BZ) - \E f(\BZ)| \geq t) \leq 2e^{-\frac{t^2}{2}}.
  \end{equation}

\end{lemma}
\begin{proof}
First note that
\[
f(\BZ) = \min_{\bs\in\mathbb{S}_+^{n-1}} \|\BZ^{\top}\bs\|.
\]

Given two Gaussian random matrices $\BZ_1$ and $\BZ_2$, and we let $\bs_1$ and $\bs_2$ be two positive unit vectors such that
\[
f(\BZ_{\ell}) = \|\BZ_{\ell}^{\top}\bs_{\ell}\|,~~\ell=1,2.
\]
Then 
\begin{align*}
\| \BZ_2^{\top}\bs_2 \| - \|\BZ_1^{\top}\bs_2\| \leq f(\BZ_2) - f(\BZ_1) & = \|\BZ_2^{\top}\bs_2\| - \|\BZ_1^{\top}\bs_1\| \leq \| \BZ_2^{\top}\bs_1 \| - \|\BZ_1^{\top}\bs_1\|
\end{align*}
and it implies
\[
|f(\BZ_2) - f(\BZ_1)| \leq \| \BZ_2 - \BZ_1 \|_F,
\]
i.e., $f(\BZ)$ is a Lipschitz-1 continuous function.
Then following from~\cite[Theorem 2.26]{W19}, we have
$\Pr( |f(\BZ) - \E f(\BZ)| \geq t) \leq 2e^{-\frac{t^2}{2}}.$
\end{proof}

\bibliography{references.bib}
\bibliographystyle{abbrv}
\end{document}